\documentclass[journal]{IEEEtran}

\usepackage{algorithm}
\usepackage{algorithmic}

\usepackage{times,amsmath,amsfonts,amssymb,subfigure,xspace,lscape,epstopdf,amsthm,bm,,mathtools,cuted} 
\usepackage{mdframed,multirow,hyperref,booktabs}
\usepackage{multicol}
\usepackage[all]{xy}
\usepackage{xcolor}

\newif\ifpdf
\ifx\pdfoutput\undefined
   \pdffalse
\else
   \pdfoutput=1
   \pdftrue
\fi
\ifpdf
   \usepackage{graphicx}
   \usepackage{epstopdf}
\else
   \usepackage{graphicx}
\fi

\usepackage{color}

\DeclareMathOperator*{\argmax}{arg\,max}
\DeclareMathOperator*{\argmin}{arg\,min}

\usepackage{changepage}
\usepackage{hyperref}

\theoremstyle{plain}
\newtheorem{thm}{Theorem}
\newtheorem{lemma}[thm]{Lemma}

\newtheorem{corollary}{Corollary}

\theoremstyle{definition}
\newtheorem*{defn*}{Definition}

\theoremstyle{remark}

\newtheorem*{note}{Note}

\begin{document}

\title{Learning to be Global Optimizer}

\author{Haotian Zhang, Jianyong Sun and Zongben Xu~\thanks{HZ, JS and ZX are all with the School of Mathematics and Statistics and National Engineering Laboratory for Big Data Analytics, Xi'an Jiaotong University, Xi'an, China. {\em Corresponding author: Jianyong Sun, email: jy.sun@xjtu.edu.cn}}}


\maketitle

\begin{abstract}
The advancement of artificial intelligence has cast a new light on the development of optimization algorithm. This paper proposes to {\em learn} a two-phase (including a minimization phase and an escaping phase) global optimization algorithm for smooth non-convex functions. For the minimization phase, a model-driven deep learning method is developed to learn the update rule of descent direction, which is formalized as a nonlinear combination of historical information, for convex functions. We prove that the resultant algorithm with the proposed adaptive direction guarantees convergence for convex functions. Empirical study shows that the learned algorithm significantly outperforms some well-known classical optimization algorithms, such as gradient descent, conjugate descent and BFGS, and performs well on ill-posed functions. The escaping phase from local optimum is modeled as a Markov decision process with a fixed escaping policy. We further propose to learn an optimal escaping policy by reinforcement learning. The effectiveness of the escaping policies is verified by optimizing synthesized functions and training a deep neural network for CIFAR image classification. The learned two-phase global optimization algorithm demonstrates a promising global search capability on some benchmark functions and machine learning tasks.
\end{abstract}


\section{Introduction}\label{introduction}

This paper considers unconstrained continuous global optimization problem:
\begin{equation}
\min\limits_{\mathbf{x}\in\mathbb{R}^n} f(\mathbf{x})\label{1}
\end{equation}where $f$ is smooth and non-convex. The study of continuous global optimization can be dated back to 1950s~\cite{dixon75}. The outcomes are very fruitful, please see~\cite{horst} for a basic reference on most aspects of global optimization,~\cite{globalopt_www} for a comprehensive archive of online information, and~\cite{pinter08} for practical applications.

Numerical methods for global optimization can be classified into four categories according to their available guarantees, namely, incomplete, asymptotically complete, complete, and rigorous methods~\cite{Neumaier04}.  We make no attempt on referencing or reviewing the large amount of literatures. Interested readers please refer to a WWW survey by Hart~\cite{gray97} and Neumaier~\cite{globalopt_www}. Instead, this paper focuses on a sub-category of incomplete method, the two-phase approach~\cite{Levy1985The,Ge1987A}.


A two-phase optimization approach is composed of a sequence of cycles, each cycle consists of two phases, a minimization phase and an escaping phase. At the minimization phase, a minimization algorithm is used to find a local minimum for a given starting point. The escaping phase aims to obtain a good starting point for the next minimization phase so that the point is able to escape from the local minimum.

\subsection{The Minimization Phase}

Classical line search iterative optimization algorithms, such as gradient descent, conjugate gradient descent, Newton method, and quasi-Newton methods like DFP and BFGS, etc., have flourished decades since 1940s~\cite{boyd2004convex,Fletcher1964Function}. These algorithms can be readily used in the minimization phase.

At each iteration, these algorithms usually take the following location update formula:
\begin{equation}x_{k+1} = x_k + \Delta_k \end{equation}where $k$ is the iteration index, $x_{k+1}, x_{k}$ are the iterates, $\Delta_k$ is often taken as $ \alpha_k \cdot d_k$ where $\alpha_k$ is the step size and $d_k$ is the descent direction. It is the chosen of $d_k$ that largely determines the performance of these algorithms in terms of convergence guarantees and rates.

In these algorithms, $d_k$ is updated by using first-order or second-order derivatives. For examples, $d_k = -\nabla f(x_k)$ in gradient descent (GD), and $-[\nabla^2 f(x_k) ]^{-1}\nabla f(x_k)$ in Newton method where $\nabla^2 f(x_k)$ is the Hessian matrix. These algorithms were usually with mathematical guarantee on their convergence for convex functions.  Further, it has been proven that first-order methods such as gradient descent usually converges slowly (with linear convergence rate), while second-order methods such as conjugate gradient and quasi-Newton can be faster (with super linear convergence rate), but their numerical performances could be poor in some cases (e.g. quadratic programming with ill-conditioned Hessian due to poorly chosen initial points).


For a specific optimization problem, it is usually hard to tell which of these algorithms is more appropriate. Further, the no-free-lunch theorem~\cite{wolpert2002no} states that ``for any algorithm, any elevated performance over one class of problems is offset by performance over another class". In light of this theorem, efforts have been made on developing optimization algorithms with adaptive descent directions.

The study of combination of various descent directions can be found way back to 1960s. For examples, the Broyden family~\cite{Sun2006Optimization} uses a linear combination of DFP and BFGS updates for the approximation to the inverse Hessian. In the Levenberg-Marquardt (LM) algorithm~\cite{marquardt1963algorithm} for nonlinear least square problem, a linear combination of the Hessian and identity matrix with non-negative damping factor is employed to avoid slow convergence in the direction of small gradients. In the accelerated gradient method and recently proposed stochastic optimization algorithms, such as momentum~\cite{Polyak1964Some}, AdaGrad~\cite{Duchi2011Adaptive}, AdaDelta~\cite{Zeiler2012ADADELTA}, ADAM~\cite{Kingma2014Adam} and such, moments of the first-order and second-order gradients are combined and estimated iteratively to obtain the location update.

Besides these work, only recently the location update $\Delta_k$ is proposed to be adaptively {\em learned} by considering it as a parameterized function of appropriate historical information: \begin{equation} \Delta_k = g(S_{k};\theta_k)\end{equation}where $S_{k}$ represents the information gathered up to $k$ iterations, including such as iterates, gradients, function criteria, Hessians and so on, and $\theta_k$ is the parameter.

Neural networks are used to model $g(S_{k}; \theta_k)$ in recent literature simply because they are capable of approximating any smooth function. For example, Andrychowicz et al.~\cite{Andrychowicz2016Learning} proposed to model $d_k$ by long short term memory (LSTM) neural network~\cite{Hochreiter1997Long} for differentiable $f$, in which the input of LSTM includes $\nabla f(x_k)$ and the hidden states of LSTM. Li et al.~\cite{li2016learning} used neural networks to model the location update for some machine learning tasks such as logistic/linear regression and neural net classifier. Chen et al.~\cite{chen2016learning} proposed to obtain the iterate directly for black-box optimization problems, where the iterate is obtained by LSTM which take previous queries and function evaluations, and hidden states as inputs. 


Neural networks used in existing learning to learn approaches are simply used as a block box. The interpretability issue of deep learning is thus inherited. A model-driven method with prior knowledge from hand-crafted classical optimization algorithms is thus much appealing. Model driven deep learning~\cite{Xu2018Model,sun2016deep} has shown its ability on learning hyper-parameters for a compressed sensing problem of the MRI image analysis, and for stochastic gradient descent methods~\cite{wang2018hyperadam,lv2017learning}.

\subsection{The Escaping Phase}

A few methods, including tunneling~\cite{levy85} and filled function~\cite{Ge1987A}, have been proposed to escape from local optimum. The tunneling method was first proposed by Levy and Montalvo~\cite{levy85}. The core idea is to use the zero of an auxiliary function, called tunneling function, as the new starting point for next minimization phase. The filled function method was first proposed by Ge and Qin~\cite{Ge1987A}. The method aims to find a point which falls into the attraction basin of a better than current local minimizer by minimizing an auxiliary function, called the filled function. The tunneling and filled function methods are all based on the construction of auxiliary function, and the auxiliary functions are all built upon the local minimum obtained from previous minimization phase. They are all originally proposed for smooth global optimization.

Existing research on tunneling and filled function is either on developing better auxiliary functions or extending to constrained and non-smooth optimization problems~\cite{ytxu15,lin11,zhang09}. In general, these methods have similar drawbacks. First, the finding of zero or optimizer of the auxiliary function is itself a hard optimization problem. Second, it is not always guaranteed to find a better starting point when minimizing the auxiliary function~\cite{zhang2004new}. Third, there often exists some hyper-parameters which are critical to the methods' escaping performances, but are difficult to control~\cite{Lan2010A}. Fourth, some proposed auxiliary functions are built with exponent or logarithm term. This could cause ill-condition problem for the minimization phase~\cite{zhang2004new}.  Last but not least, it has been found that though the filled and tunneling function methods have desired theoretical properties, their numerical performance is far from satisfactory~\cite{zhang2004new}.

\subsection{Main Contributions}

In this paper, we first propose a model-driven learning approach to learn adaptive descent directions for locally convex functions. A local-convergence guaranteed algorithm is then developed based on the learned directions. We further model the escaping phase within the filled function method as a Markov decision process (MDP) and propose two policies, namely a fixed policy and a policy learned by policy gradient, on deciding the new starting point. Combining the learned local algorithm and the escaping policy, a two-phase global optimization algorithm is finally formed.

We prove that the learned local search algorithm is convergent; and we explain the insight of the fixed policy which can has a higher probability to find promising starting points than random sampling. Extensive experiments are carried out to justify the effectiveness of the learned local search algorithm, the two policies and the learned two-phase global optimization algorithm.

The rest of the paper is organized as follows. Section~\ref{rl} briefly discusses the reinforcement learning and policy gradient to be used in the escaping phase. Section~\ref{l2local} presents the model-driven learning to learn approach for convex optimization. The escaping phase  is presented in Section~\ref{escape}, in which the fixed escaping policy under the MDP framework is presented in Section~\ref{escaping_MDP}, while the details of the learned policy is presented in Section~\ref{mechanism_RL}. Controlled experimental study is presented in Section~\ref{experiments}. Section~\ref{discussion} concludes the paper and discusses future work.

\section{Brief Introduction of Reinforcement Learning}\label{rl}




In reinforcement learning (RL), the learner (agent) chooses to take an action at each time step; the action changes the state of environment; (possibly delayed) feedback (reward) returns as the response of the environment to the learner's action and affects the learner's next decision. The learner aims to find an optimal policy so that the actions decided by the policy maximize cumulative rewards along time.
\begin{figure}
\centering\includegraphics[width=\columnwidth]{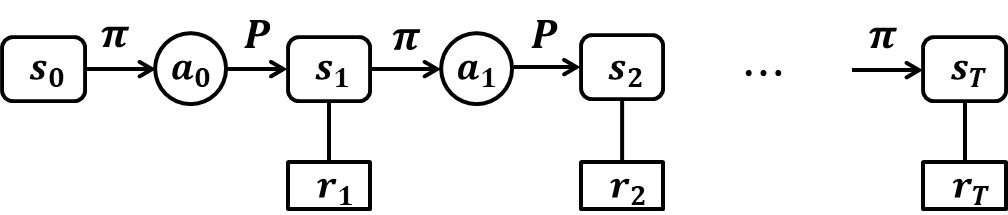}\centering
\caption{Illustration of a finite horizon Markov decision process. }\label{RL_introduction}
\end{figure}

Consider a finite-horizon MDP with continuous state and action space defined by the tuple $(\mathcal{S},\mathcal{A},\mu_0,p,r,\pi,T)$ where $\mathcal{S}\in\mathbb{R}^D$ denotes the state space, $\mathcal{A}\in\mathbb{R}^d$ the action space, $\mu_0$ the initial distribution of the state, $r:\mathcal{S}\rightarrow\mathbb{R}$ the reward, and $T$ the time horizon, respectively. At each time $t$, there are $s_t\in\mathcal{S}$, $a_t\in\mathcal{A}$ and a transition probability $p:\mathcal{S}\times\mathcal{A}\times\mathcal{S}\rightarrow \mathbb{R}$ where $p(s_{t+1}|a_t,s_t)$ denotes the transition probability of $s_{t+1}$ conditionally based on $s_t$ and $a_t$. The policy $\pi:\mathcal{S}\times\mathcal{A}\times \{0,1,\cdots\,T\}\rightarrow\mathbb{R}$, where $\pi(a_t|s_t;\theta)$ is the probability  of choosing action $a_t$ when observing current state $s_t$ with $\theta$ as the parameter.


As shown in Fig.~\ref{RL_introduction}, starting from a state $s_0\sim \mu_0$, the agent chooses $a_0\sim \pi(a_0|s_0,\theta)$; after executing the action, agent arrives at state $s_1\sim p(s_{1}|a_0,s_0)$. Meanwhile, agent receives a reward $r(s_1)$ (or $r_1$) from the environment. Iteratively, a trajectory $\tau=\{s_0,a_0,r_1, s_1,a_1,r_2, \cdots,a_{T-1},s_T, r_T\}$ can be obtained. The optimal policy $\pi^*$ is to be found by maximizing the expectation of the cumulative reward $R(\tau) = \left[\sum_{t=0}^{T-1} \gamma^t r(s_{t+1})\right]$: \begin{equation}
\pi^*=\arg\max\limits_{\pi}\mathbb{E}_{\tau} \left[ R_{\tau} \right] = \sum_{\tau} q(\tau;\theta) R(\tau)\triangleq U(\theta) \label{RL_aim}
\end{equation}where the expectation is taken over trajectory $\tau\sim q(\tau;\theta)$ where
\begin{equation}
q(\tau;\theta)=\mu_0(s_0)\prod_{t=0}^{T-1}\pi(a_t|s_t,\theta)p(s_{t+1}|a_t,s_t).\label{tau_density}
\end{equation}

A variety of reinforcement learning algorithms have been proposed for different scenarios of the state and action spaces, please see~\cite{arulkumaran2017deep} for recent advancements. The RL algorithms have succeeded overwhelmingly for playing games such as GO~\cite{silver2016mastering}, Atari~\cite{mnih2013playing} and many others.

We briefly introduce the policy gradient method for continuous state space~\cite{Sutton1998Reinforcement}, which will be used in our study. Taking derivative of $U(\theta)$ w.r.t. $\theta$ discarding unrelated terms, we have
\begin{eqnarray}\label{policy_gradient}
\nabla U(\theta) &=&  \sum_{\tau}R(\tau) \nabla_{\theta} q(\tau;\theta) \nonumber\\
&= & \sum_{\tau}R(\tau) \nabla_{\theta} \log q(\tau;\theta) q(\tau;\theta) \nonumber\\
&=& \sum_{\tau} R(\tau) q(\tau;\theta)  \left[\sum_{t=0}^{T-1}\nabla_{\theta} \log \pi(a_t|s_t,\theta)\right] \label{policy_gradient}
\end{eqnarray}Eq.~\ref{policy_gradient} can be calculated by sampling trajectories $\tau_1, \cdots, \tau_N$ in practice:
\begin{eqnarray}
\nabla U(\theta)\approx\frac{1}{N}\sum_{i=1}^{N}\sum_{t=0}^{T-1}\nabla_{\theta} \log \pi(a_t^{(i)}|s_t^{(i)},\theta) R(\tau^i)\label{policy_gradient_sampling}
\end{eqnarray}where $a_t^{(i)} (s_t^{(i)})$ denotes action (state) at time $t$ in the $i$th trajectory, $R(\tau^i)$ is the cumulative reward of the $i$th trajectory.

For continuous state and action space, normally assume
\begin{eqnarray}
\pi(a|s,\theta)=\frac{1}{\sqrt{2\pi\sigma}}\exp\left\{-\frac{(a-\phi(s;\theta))^2}{\sigma^2}\right\}\label{policy_form}
\end{eqnarray}where $\phi$ can be any smooth function, like radial basis function, linear function, and even neural networks. 

\section{Model-driven Learning to Learn for Local Search}\label{l2local}


In this section, we first summarize some well-known first- and second-order classical optimization algorithms. Then the proposed model-driven learning to optimize method for locally convex functions is presented.

\subsection{Classical Optimization Methods}

In the sequel, denote ${g_k}=\nabla f(x_k)$, ${s_k}= x_{k+1}-x_{k}$, ${y_k}={g_{k+1}}-{g_{k}}$. The descent direction $d_k$ at the $k$-th iteration of some classical methods is of the following form~\cite{Sun2006Optimization}:
\begin{equation}
d_k= \left\{
 \begin{array}{ll}
 -g_k					& \text{steepest GD} \\
 -g_k+ \alpha_k d_{k-1}	& \text{conjugate GD} \\
 -H_{k}g_k				& \text{quasi-Newton}
\end{array}
 \right.
\end{equation}where $H_k$ is an approximation to the inverse of the Hessian matrix, and $\alpha_k$ is a coefficient that varies for different conjugate GDs. For example, $\alpha_k$ could take $g_k^\intercal y_{k-1} / d_{k-1}^\intercal y_{k-1}$ for Crowder-Wolfe conjugate gradient method~\cite{Sun2006Optimization}.

The update of $H_k$ also varies for different quasi-Newton methods. In the Huang family, $H_{k}$ is updated as follows:
\begin{equation} \label{hupdate}
H_{k}=H_{k-1}+s_{k-1} u_{k-1}^\intercal+ H_{k-1}y_{k-1} v_{k-1}^\intercal
\end{equation}where
\begin{eqnarray}
u_{k-1}&=&a_{11}s_{k-1}+a_{12}H_{k-1}^\intercal y_{k-1}\\
v_{k-1}&=&a_{21}s_{k-1}+a_{22}H_{k-1}^\intercal y_{k-1}\\
u_{k-1}^\intercal y_{k-1} &=& \rho, v_{k-1}^\intercal y_{k-1}=-1
\end{eqnarray}The Broyden family is a special case of the Huang family in case $\rho=1$, and $a_{12} = a_{21} $.



\subsection{Learning the descent direction: d-Net}

We propose to consider the descent direction $d_k$ as a nonlinear function of $S_k = \left\{g_k, g_{k-1}, s_{k-1}, s_{k-2}, y_{k-2}\right\}$ with parameter $\theta_k = \{ w_k^1,w_k^2,w_k^3,w_k^4, \beta_k \}$ for the adaptive computation of descent search direction ${d_k}=h(S_k; \theta_k)$. Denote
{\begin{eqnarray}R_{k-1}&=&\mathbf{I}-\frac{{s_{k-1}}({w_k^1{g_k}-w_k^2{g_{k-1}}})^\intercal}{{s_{k-1}^\intercal}({w_k^3{g_k}-w_k^4{g_{k-1}}})}.\label{rk}\end{eqnarray}}We propose
{\begin{equation} h(S_k; \theta_k)=-R_{k-1}\Big({\beta}_k{H_{k-1}}+(1-{\beta}_k)\mathbf{I}\Big){g_k}\label{hgd}
\end{equation}}where $\mathbf{I}$ is the identity matrix.

At each iteration, rather than updating $H_{k-1}$ directly, we update the multiplication of $R_{k-1}$ and $H_{k-1}$ like in the Huang family ~\cite{Sun2006Optimization}:
\begin{eqnarray}\label{rhupdate}
R_{k-1}H_{k-1}=R_{k-1}R_{k-2}H_{k-2}+\rho R_{k-1}\frac{{s_{k-2}}{s_{k-2}}^\intercal}{{s_{k-2}}^\intercal{y_{k-2}}}.
\end{eqnarray}

It can be seen that with different parameter $w_k^i, i = 1,\cdots,4$ and $\beta_k$ settings, $d_k$ can degenerate to different directions:
\begin{itemize}
\item when ${w_k^1,w_k^2,w_k^3,w_k^4\in \{0,1\}}$, the denominator of $R_k$ is not zero, and ${\beta_k=0}$, the update degenerates to conjugate gradient.
\item when ${w_k^1,w_k^2,w_k^3,w_k^4\in \{0,1\}}$, and the denominator of $R_k$ is not zero, and ${\beta_k=1}$, the update becomes the preconditioned conjugate gradient.
\item when ${w_k^1=1},{w_k^2=1},{w_k^3=1},{w_k^4=1}$, and ${\beta_k=1}$, the update degenerates to the Huang family.
\item when ${w_k^1=0},{w_k^2=0}$, the denominator of $R_k$ is not zero, and ${\beta_k=0}$ the update becomes the steepest GD.
\end{itemize}

Based on Eq.~\ref{hgd}, a new optimization algorithm, called adaptive gradient descent algorithm (AGD), can be established. It is summarized in Alg.~\ref{alg:HGD}. It is seen that to obtain a new direction by Eq.~\ref{hgd}, information from two steps ahead is required as included in $S_k$. To initiate the computation of new direction, in Alg.~\ref{alg:HGD}, first a steep gradient descent step (lines~\ref{sg1}-\ref{sg3}) and then a non-linear descent step (lines~\ref{sg4}-\ref{sg7}) are applied. With these prepared information, AGD iterates (lines~\ref{agd1}-\ref{agd2}) until the norm of gradient at the solution $x_k$ is less than a positive number $\epsilon$.
\textcolor{red}{
\begin{algorithm}\caption{The adaptive gradient descent algorithm (AGD)}
\label{alg:HGD}
\begin{algorithmic}[1]
\STATE initialize $x_0$, ${H_0}\leftarrow I$ and $\epsilon>0$;\
\STATE \textit{\# a steep gradient descent step}
\STATE $g_0\leftarrow \nabla f(x_0), d_0\leftarrow g_0$ \label{sg1}
\STATE Choose ${\alpha}_0$ through line search;\  \label{sg2}
\STATE $x_1\leftarrow x_0-{{\alpha}_0}g_0$; \label{sg3}
\STATE \textit{\# a non-linear descent step}
\STATE $g_1\leftarrow \nabla f(x_1)$, $s_{0}\leftarrow x_{1}-x_{0}$; \label{sg4}
\STATE Compute $R_0 \leftarrow \mathbf{I}-\frac{{s_{0}}({w_1^1{g_1}-w_1^2{g_{0}}})^\intercal}{{s_{0}^\intercal}({w_1^3{g_1}-w_1^4{g_{0}}})} $ and $d_1\leftarrow R_0H_0g_1$;\label{sg5}
\STATE Choose ${\alpha}_1$ through line search;\ \label{sg6}
\STATE $x_2\leftarrow x_1-{{\alpha}_1}d_1$;\label{sg7}
\STATE Set $k\leftarrow 2$;
\REPEAT
\STATE Compute $g_k\leftarrow \nabla f(x_k)$, $s_{k-1}\leftarrow x_{k}-x_{k-1}$, $y_{k-2}\leftarrow g_{k-1}-g_{k-2}$;\;
\STATE Gather $S_k\leftarrow \left\{g_k, g_{k-1}, s_{k-1}, s_{k-2}, y_{k-2}\right\}$;\; \label{agd1}
\STATE Compute ${d_k}\leftarrow h(S_k; \theta_k)$;\;
\STATE Choose ${\alpha}_k$ through line search;\
\STATE Update $R_{k}H_{k}$;
\STATE  $x_{k+1}\leftarrow x_k-{{\alpha}_k}d_k$;\ \label{agd2}
\STATE  $k\leftarrow k+1$;\
\UNTIL{$\|g_k\|\leq \epsilon$}
\end{algorithmic}
\end{algorithm}}

To specify the parameters $\theta_k$ in the direction update function $h$, like~\cite{Andrychowicz2016Learning}, we unfold AGD into $T$ iterations. Each iteration can be considered as a layer in a neural network. We thus have a  `deep' neural network with $T$ layers. The resultant network is called {\em d-Net}. Fig.~\ref{dnet} shows the unfolding.

Like normal neural networks, we need to train for its parameters $\theta = \{\theta_1, \cdots, \theta_T\}$. To learn the parameters, the loss function $\ell({\theta})$ is defined as
\begin{equation}\label{loss}
\begin{split}
L({\theta})&=\mathbb{E}_{f\in {\cal F}}\left( \sum_{t=1}^T{f(x_t)}\right) \\
		&=\mathbb{E}_{f\in {\cal F}}\left(\sum_{t=2}^T{f\Big(x_{t-1}+{\alpha_t}h(S_{t-1};\theta_{t-1})\Big)}\right)
\end{split}
\end{equation}That is, we expect these parameters are optimal not only to a single function, but to a class of functions $\cal F$; and to all the criteria along the $T$ iterations.

We hereby choose $\cal F$ to be the Gaussian function family:
\begin{equation} {\cal F}_{\cal G}  = \left\{ f \big| f(x)=\exp\left(-x^\intercal \Sigma^{-1} x\right), \Sigma \succeq 0 , x \in \mathbb{R}^n \right\} \end{equation}There are two reasons to choose the Gaussian function family. First, any $f \in {\cal F}_{\cal G}$ is locally convex. That is, let $H(x)$ represents the Hessian matrix of $f(x)$, it is seen that
\begin{equation*}
\lim\limits_{x\rightarrow{0}}{H(x)=f(x){\Sigma}^{-1}+{\Sigma}^{-1}xx^\intercal{f(x)}={\Sigma}^{-1}}\succeq 0
\end{equation*}Second, it is known that finite mixture Gaussian model can approximate a Riemann integrable function with arbitrary accuracy~\cite{Wilson2000Multiresolution}. Therefore, to learn an optimization algorithm that guarantees convergence to local optima, it is sufficient to choose functions that are locally convex.

 Given ${\cal F}$, when optimizing $\ell(\theta)$, the expectation can be obtained by Monte Carlo approximation with a set of functions sampled from ${\cal F}$, that is
 \[\ell(\theta) \approx \frac{1}{N} \sum_{i = 1}^N \sum_{t=2}^T f_i\Big(x_{t-1}+{\alpha_t}h(S_{t-1};\theta_{t-1})\Big)\]where $f_i \sim {\cal F}$. $\ell(\theta)$ can then be optimized by the steepest GD algorithm.

\textbf{Note:} The contribution of the proposed d-Net can be summarized as follows. First, there is a significant difference between the proposed learning to learn approach with existing methods, such as~\cite{Andrychowicz2016Learning,chen2016learning}. In existing methods, LSTM is used as a `black-box' for the determination of descent direction, and the parameters of the used LSTM is shared among the time horizon. Whereas in our approach, the direction is a combination of known and well-studied directions, i.e. a `white-box', which means that our model is interpretable. This is a clear advantage against black-box models.

Second, in classical methods, such as the Broyden and Huang family and LM, descent directions are constructed through a linear combination. On the contrary, the proposed method is nonlinear and subsumes a wide range of classical methods. This may result in better directions.

Further, the combination parameters used in classical methods are considered to be hyper-parameters. They are normally set by trial and error. In the AGD, these parameters are learned from the optimization experiences to a class of functions, so that the directions can adapt to new optimization problem.

\begin{figure}
\centering\includegraphics[width=\columnwidth]{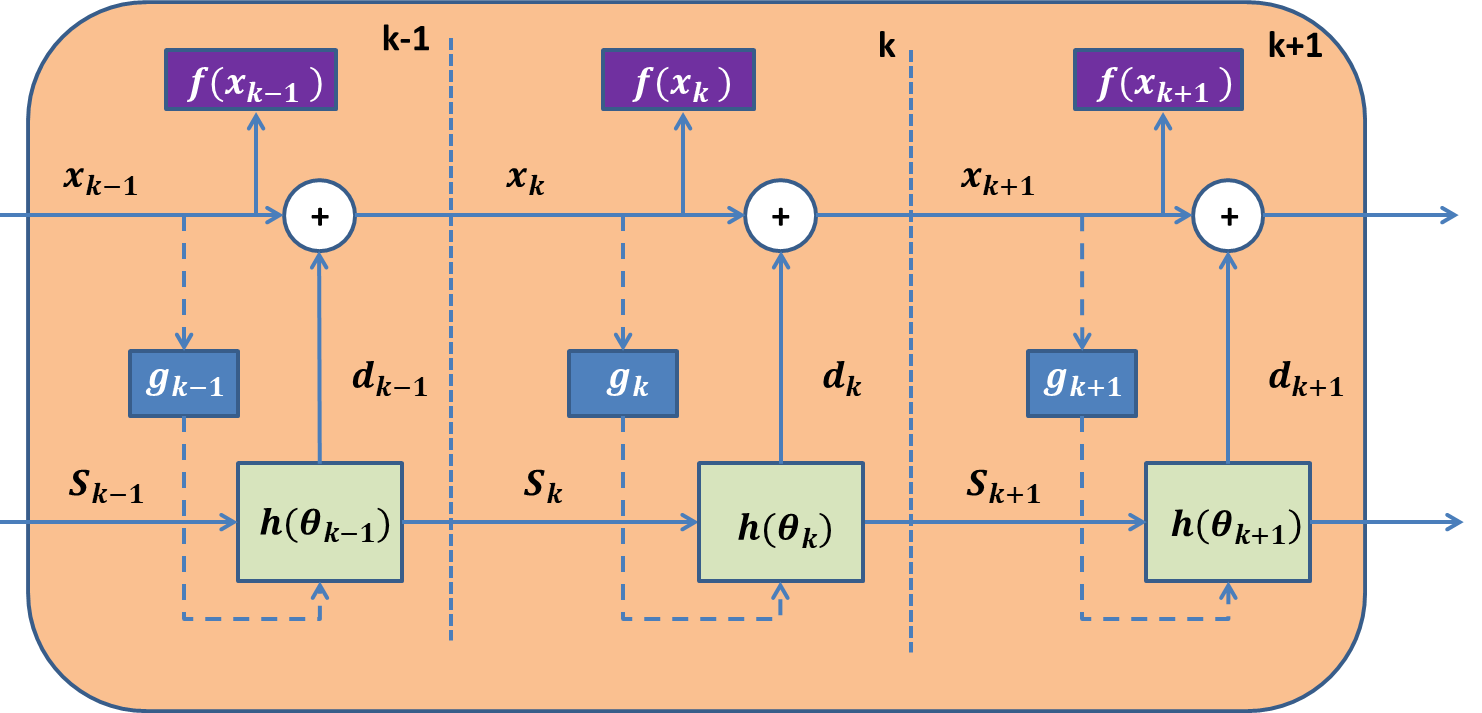}\centering
\caption{The unfolding of the AGD. }\label{dnet}
\end{figure}

\subsection{Group d-Net}

To further improve the search ability of d-Net, we employ a group of d-Nets, dubbed as Gd-Net. These d-Nets are connected sequentially, with shared parameters among them. Input of the $k$-th ($k > 1$) d-Net is the gradient from $(k-1)$-th d-Net. To apply Gd-Net, an initial point is taken as the input, and is brought forward through these d-Nets until the absolute gradient norm is less than a predefined small positive real number.


In the following we show that Gd-Net guarantees convergence to optimum for convex functions. We first prove that AGD is convergent. Theorem~\ref{dNetConv} summarizes the result. Please see Appendix A for proof.

\begin{thm}\label{dNetConv}
Assume $f:\mathbb{R}^n \rightarrow \mathbb{R}$ is continuous and differentiable and the sublevel set
\[L(x)=\Big\{x\in \mathbb{R}^n | f(x)\leq f(x_0) \Big\} \text{\ for any\ } x_0 \in \mathbb{R}^n\] is bounded. The sequence $\{ x_k, k = 1, 2,\cdots \}$ obtained by AGD with exact line search converges to a stable point.
\end{thm}

Since d-Net is the unfolding of AGD, from Theorem~\ref{dNetConv}, it is sure that the iterate sequence obtained by d-Net is non-increasing for any initial $x_0$ with properly learned parameters. Therefore, applying a sequence of d-Net (i.e. Gd-Net) on a bound function $f(x)$ from any initial point $x_0$ will result in a sequence of non-increasing function values. This ensures that the convergence of the sequence, which indicates that Gd-Net is convergent under the assumption of Theorem~\ref{dNetConv}.


\section{Escaping from local optimum}\label{escape}

Gd-Net guarantees convergence for locally convex functions. To approach global optimality, we present a method to escape from the local optimum once trapped. Our method is based on the filled-function method, and is embedded within the MDP framework.

\subsection{The Escaping Phase in the Filled Function Method}

In the escaping phase of the filled function method, a local search method is applied to minimize the filled function for a good starting point for next minimization phase. To apply the local search method, the starting point is set as $x_0+\delta_0 d$ where $x_0$ is the local minimizer obtained from previous minimization phase, $\delta_0$ is a small constant and $d$ is the search direction.

Many filled functions have been constructed (please see~\cite{zhang2004new} for a survey). One of the popular filled-functions~\cite{Ge1987A} is defined as follows
\begin{equation}\label{filled_function}
 H(x)=-\exp(a\|x-x_0\|^2)({f(x)-f(x_0)})
\end{equation}where $a$ is a hyper-parameter. It is expected that minimizing $H(x)$ can lead to a local minimizer which is away from $x_0$ due to the exist of the exponential term.




Theoretical analysis has been conducted on the filled function methods in terms of its escaping ability~\cite{Ge1987A}. However, the filled function methods have many practical weaknesses yet to overcome.

First, the hyper-parameter $a$ is critical to algorithm performance. Basically speaking, if $a$ is small, it struggles to escape from $x_0$, otherwise it may miss some local minima. But it is very hard to determine the optimal value of $a$. There has no theoretical results, neither rule of thumb on how to choose $a$.

Second, the search direction $d$ is also very important to the algorithmic performance. Different $d$'s may lead to different local minimizers, and the local minimizers are not necessarily better than $x_0$. In literature, usually a trial-and-error procedure is applied to find the best direction from a set of pre-fixed directions, e.g. along the coordinates~\cite{Ge1987A}. This is apparently not effective. To the best of our knowledge, no work has been done in this avenue.

Third, minimizing $H(x)$ itself is hard and may not lead to a local optimum, but a saddle point~\cite{Ge1987A} even when a promising search direction is used. Unfortunately, there is no studies on how to deal with this scenario in literature. Fig.~\ref{filled} shows a demo about this phenomenon. In the figure, the contour of $f(x)$ is shown in red lines, while the negative gradients of the filled function $H(x)$ are shown in blue arrows. From Fig.~\ref{filled}, it is seen that minimizing $H(x)$ from a local minimizer of $f$ at $(4,13)$ will lead to the saddle point at $(12,15)$.






\begin{figure}
\centering\includegraphics[width=0.8\columnwidth]{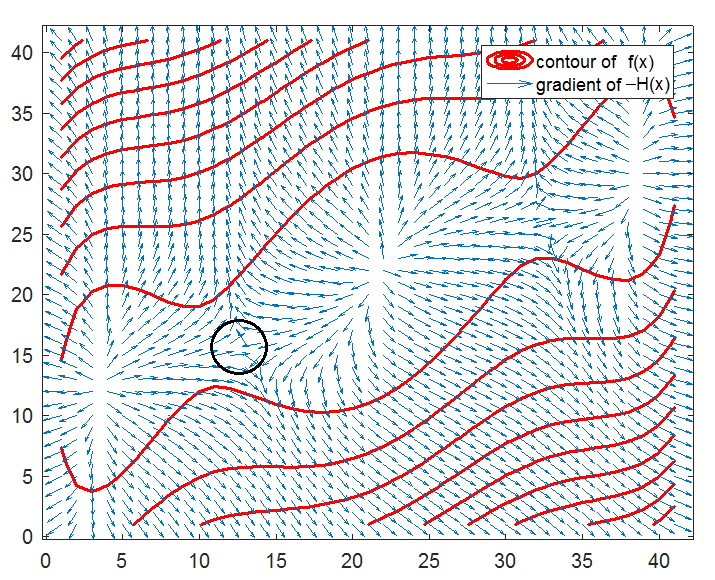}
\caption{Red lines show the contour of the three-hump function $f(x_1, x_2)$, blue arrows are the gradients of $-H(x_1, x_2)$. There is a saddle point at (12,15) for $-H(x_1, x_2)$ .}\label{filled}
\end{figure}


\subsection{The Proposed Escaping Scheme}\label{escaping_MDP}


The goal of an escaping phase is to find a new starting point $x^{\text{new}} =x^{\text{old}} + \Delta x$ such that $ x^{\text{new}}$ can escape from the attraction basin of $x^{\text{old}}$ (the local minimizer obtained from previous minimization phase) if a minimization procedure is applied, where $\Delta x = \delta d$, $d$ is the direction and $\delta$ is called the escaping length in this paper.

Rather than choosing $d$ from a pre-fixed set, we could sample some directions, either randomly or sequentially following certain rules. In this section, we propose an effective way to sample directions, or more precisely speaking $\Delta x$'s.



In our approach, the sampling of $\Delta x$ is modeled as a finite-horizon MDP. That is, the sampling is viewed as the execution of a policy $\pi$: at each time step $t$, given the current state $s_t$, and reward $r_{t+1}$, an action $a_t$, i.e. the increment $\Delta x$, is obtained by the policy. The policy returns $\Delta x $ by deciding a search direction $d_t$ and an escaping length $\delta_t$.

At each time step $t$, the state $s_t$ is composed of a collection of previously used search directions $\mathbf{d}_t = \{d_1, \cdots, d_{N_0}\}$ and their scores $\mathbf{u}_t = \{u_1, \cdots, u_{N_0}\}$, where $N_0$ is a hyper-parameter. Here the score of a search direction measures how promising a direction is in terms of the quality of the new starting point that it can lead to. A new starting point is of high quality if applying local search from it can lead to a better minimizer than current one. The initial state $s_0$ includes a set of $N_0$ directions sampled uniformly at random, and their corresponding scores.


In the following, we first define `score', then present the policy on deciding $a_t$ and $\delta_t$, and the transition probability $p(s_{t+1}|a_t, s_t)$. Without causing confusion, we omit the subscript in the sequel.




\subsubsection{Score} Given a search direction $d$, a local minimizer $x_0$, define \begin{equation}\label{rewardx} u_d(t)= -\nabla f(x_0+t\cdot d)^\intercal{d}\end{equation}where $t \in \mathbb{R}_{++}$ is the step size along $d$.


Since $x_0$ is a local minimizer, by definition, we cannot find a solution with smaller $f$ along $d$ if $x_0+td$ is within the attraction basin of $x_0$. However, if there is a $T$ such that $ x'=x_0+T\cdot d$ is a point with smaller criterion than $x_0$ (i.e. $f(x') < f(x_0)$), and there is no other local minimizer within ${\cal B}_0 = \{x \in \mathbb{R}^n | \|x - x_0\|_2 \leq \|x'-x_0\|_2\}$, we can prove that there exists a $\xi$, such that $u_d(t)<0 $ when $t\in [0,\xi)$, and $u_d(t)>0$ when $t\in (\xi,T]$ (proof can be seen in Appendix B). Theorem~\ref{reward} summarizes the result under the following assumptions:
\begin{itemize}
\item[(1)] $f(x)\in C^2(\mathbb{R}^n)$ has finite number of local minimizer.
\item[(2)] For every local optimum $x_{\ell}$, there exists a $r$ such that $f(x)$ is convex in ${\cal B}(x_{\ell},r) =\{x: \|x - x_{\ell}\|_2 \leq r\}$.
\item[(3)] The attractive basin of each local optimum is convex.
\end{itemize}
\begin{thm}\label{reward}
If $x'=x_0+T\cdot d$ is a point outside the boundary of $x_0$'s attraction basin, there is no other local minimizer within ${\cal B}_0 = \{x \in \mathbb{R}^n | \|x - x_0\|_2 \leq \|x'-x_0\|_2\}$. Then there exists a $\xi$ such that \[g(t)\triangleq f(x_0+t\cdot d),t\in [0,T]\]obtains its maximum at $\xi$. And $g(t)$ is monotonically increasing in $[0,\xi)$, and monotonically decreasing in $(\xi,T]$.
\end{thm}

If we let $d = x' - x_0$, then $g'(t)=\nabla f(x_0+t(x'-x_0))^\intercal{(x'-x_0)}\triangleq-u_d(t)$. This implies that $u_d(t)$ is actually $- g'(t)$ along the direction from $x'$ pointing to $x_0$. This tells whether a direction $d$ can lead to a new minimizer or not. A direction $d$ with a positive $u_d(t)$ indicates that it could lead to a local minimizer different to present one.

We therefore define the score of a direction $d$, $u_d$, to be the greatest $u_d(t)$ along $d$, i.e.  \begin{equation} u_d \triangleq \max\limits_{t \in[0,T]} u_d(t)\label{scores}\end{equation} For such $d$ that $u_d > 0$, we say it is promising. 





In the following, we present the policy $\pi$ on finding $\Delta x $ (or new starting point). The policy includes two sub-policies. One is to find the new point given a promising direction, i.e. to find the escaping length. The other is to decide the promising direction.

\subsubsection{Policy on finding the escaping length}


First we propose to use a simple filled function as follows:
\begin{equation} \widetilde{H}(x) = - a \|x-x_0\|^2.\end{equation}Here $a$ is called the `escaping length controller' since it controlls how far a solution could escape from the current local optimum. Alg.~\ref{alg:distance determination} summarizes the policy proposed to determine the optimal $a^*$ and the new starting point $x$. 


\begin{algorithm}\caption{Policy on finding a new starting point }\label{alg:distance determination}
\begin{algorithmic}[1]
\REQUIRE a local minimum $x_0$, a direction $d$, a bound $M > 0$, an initial escaping length controller $a>0$, a learning rate $\alpha$, some constants $\delta_0>0$, $N\in \mathbb{Z}$ and $\epsilon > 0$\\
\ENSURE a new starting point $x$ and $u_d$ (the score of $d$)
\REPEAT
\STATE set $x_1=x_0+\delta_0 d$;
\STATE optimize $\widetilde{H}(x)$ along $d$ starting from $x_1$ for $N$ iterations, i.e. evaluate the criteria of a sequence of $N$ points defined by $x_j = x_{j-1}+\alpha 2a (x_{j-1} - x_0), 2\leq j \leq N$;\label{a1}
\STATE compute $F(a) \leftarrow \sum_{j=2}^N {f(x_j)/(j-1) } $;\label{a2}
\STATE $a \leftarrow a + \alpha F'(a)$;\label{a3}
\UNTIL{$|F'(a)| \leq \epsilon$ or $\|x_N - x_0\| \geq M$.}
\STATE $u_d \leftarrow  \max_{i=1, \cdots, N} -\nabla f(x_i)^\intercal d$;
\STATE if $\|x_N - x_0\| \geq M$, then set $u_d \leftarrow -|u_d|$;\label{stop}
\STATE \textbf{return}  $u_d$ and $x\leftarrow x_N$.
\end{algorithmic}
\end{algorithm}

In Alg.~\ref{alg:distance determination}, given a direction $d$, the filled function $\widetilde{H}(x)$ is optimized for $N$ steps (line~\ref{a1}). The sum of the iterates' function values, denoted as $F(a)$ (line~\ref{a2}), is maximized w.r.t. $a$ by gradient ascent (line~\ref{a3}). The algorithm terminates if a stable point of $F(a)$ is found ($|F'(a)| \leq \epsilon$), or the search is out of bound ($\|x_N - x_0\| \geq M$). When the search is out of bound, a negative score is set for the direction $d$ (line~\ref{stop}). As a by-product, Alg.~\ref{alg:distance determination} also returns the score $u_d $ of the given direction $d$. 

We prove that $x_N$ can escape from the attraction basin of $x_0$ and ends up in another attraction basin of a local minimizer $x'$ with smaller criterion {\em if $x'$ exists}. Theorem~\ref{existence} summarizes the result. 

%


\begin{thm}\label{existence}
Suppose that $x' =x_0+Td$ is a point such that $f(x_0)\geq f(x')$, and there are no other points that are with smaller or equal criterion than $f(x_0)$ within ${\cal B}_0$. If the learning rate $\alpha$ is sufficiently small, then there exists an $a^*$ such that $F'(a^*)=0$.
\end{thm}

According to this theorem, we have the following corollary.
\begin{corollary}\label{cor1}
Suppose that $x_N$ is the solution obtained by optimizing $\widetilde{H}(x) = - a^*\|x-x_0\|_2^2$ along $d$ starting from $x_0+\delta_0 d$ at the $N$-th iteration, then $x_N$ will be in an attraction basin of $x'$, if the basin ever exists.
\end{corollary}

Theorem~\ref{existence} can be explained intuitively as follows. Consider pushing a ball down the peak of a mountain with height $-f(x_1)$ (it can be regarded as the ball's gravitational potential energy) along a direction $d$. The ball will keep moving until it arrives at a point $\tilde{x}=x_0+\tilde{t}d$ for some $\tilde{t}$ such that $f(x_1)=f(\tilde{x})$. For any $t\in[\delta_0,\tilde{t})$, the ball has a positive velocity, i.e. $g(t)-g(t_1) > 0$ where $t_1 = \delta_0$. But the ball has a zero velocity at $\tilde{t}$, and negative at $t > \tilde{t}$, Hence $\int (f(x_0+td)-f(x_1))dt$ reaches its maximum in $[0, \tilde{t}]$. The integral is approximated by its discrete sum, i.e. $F(a)$, in Alg.~\ref{alg:distance determination}.


Further, according to the law of the conservation of energy, the ball will keep moving until at some $\tilde{t}$, $f(x_0+td)- f(x_1)= 0$ in which case $F'(a) = 0$. This means that the ball falls into the attraction basin of a smaller criterion than $f(x_1)$ as shown in Fig.~\ref{explain2}(b). Fig.~\ref{explain2}(a) shows when $a$ is small, in $N$ iterations, the ball reaches some $t_N$ but $F'(a)\approx\int_{t_1}^{t_N} (g'(t) )> 0$.

Moreover, if there is no smaller local minimizers in search region, the ball will keep going until it rolls outside the restricted search region bounded by $M$ as shown in Fig.~\ref{explain2}(e) which means Alg.~\ref{alg:distance determination} fails to find $a^*$. Fig.~\ref{explain2}(c)(d) show the cases when there are more than one local minimizers within the search region.

\begin{figure*}
\centering
\includegraphics[scale=0.29]{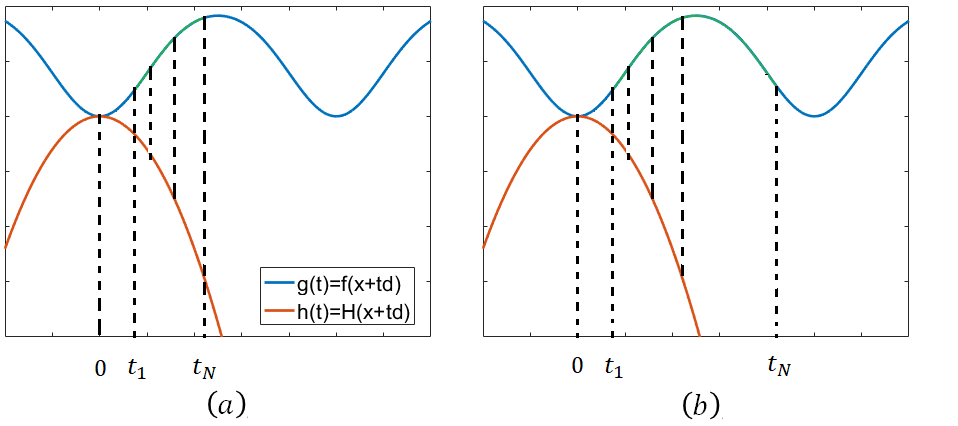}
\includegraphics[scale=0.28]{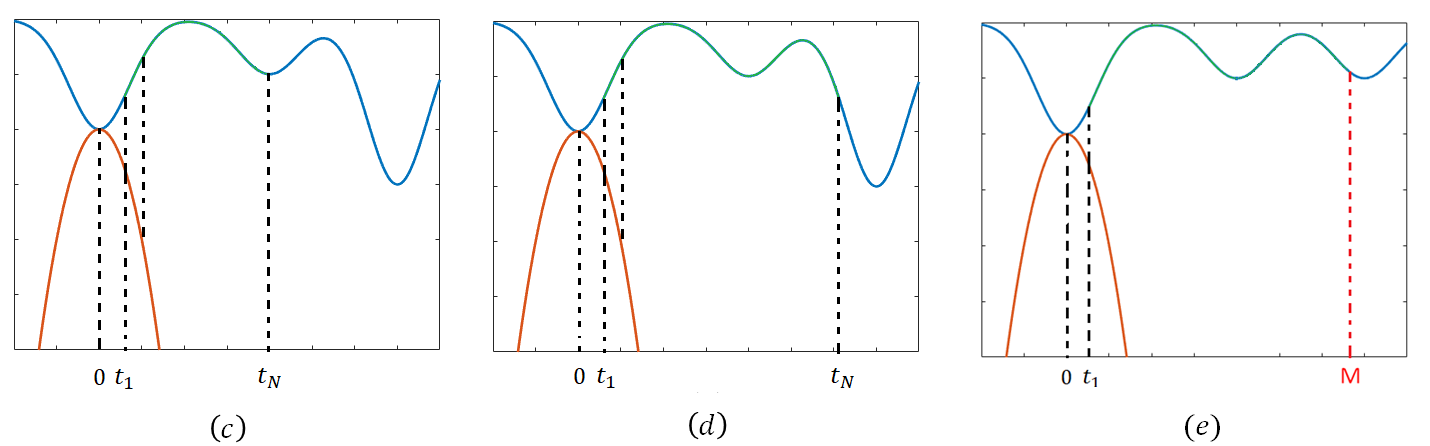}\centering
\caption{Possible scenarios encountered when estimating $a^*$. (a) shows the case when $a$ is not large enough, while (b) shows when $a$ is appropriate. (c) shows that $x_N = x_0 + t_N d$ reaches a local minimum, but $F'(a)\neq 0$ because $g(t_N) -g(t_1) \neq 0$; (d) shows the case when there are more than one local minimizer. (e) shows when there is no smaller local minimizer within $\|x_N - x_0\| \geq M$.}\label{explain2}
\end{figure*}

Once such $a^*$ has been found, the corresponding $x_N$ will enter an attraction basin of a local minimum with smaller criterion than $x_0$. If we cannot find such an $a^*$ in the direction of $d$ within a distance $M$ to $x_0$, we consider that there is no another smaller local minimum along $d$. If it is the case, $d$ is non-promising. We thus set a negative score for it as shown in line~\ref{stop} of Alg.~\ref{alg:distance determination}. 


It is seen that the running of line~\ref{a3} of Alg.~\ref{alg:distance determination} requires to compute $N$ gradients of $f(x)$ at each iteration. This causes Alg.~\ref{alg:distance determination} time consuming. We hereby propose to accelerate this procedure by fixing $a$ but finding a proper number of iterations.  Alg.~\ref{accelerate} summarizes the fast policy. Given a direction $d$, during the search, the learning rate $\alpha$ and the escaping length controller $a$ are fixed. At each iteration of Alg.~\ref{accelerate}, an iterate $x_i$ is obtained by applying gradient descent over $\widetilde{H}(x)$. The gradient of $x_i$ over $f(x)$ is computed (line~\ref{c1}). $Q_i = \sum_{j=1}^i \nabla f(x_j)^\intercal (x_j - x_{j-1})$ is computed (line~\ref{c2}). Alg.~\ref{accelerate} terminates if there is an $i$, such that $Q_i > 0 \text{\ and\ } Q_{i-1} < 0$ or the search is beyond the bound. It is seen that during the search, at each iteration, we only need to compute the gradient for once, which can significantly reduce the computational cost in comparison with Alg.~\ref{alg:distance determination}.

\begin{algorithm}\caption{Fast policy on finding a new starting point}\label{accelerate}
\begin{algorithmic}[1]
\REQUIRE a local minimizer $x_0$, a search direction $d$, a bound $M > 0$ and positive scalars $a, \delta_0, \epsilon$ and $\alpha$;
\ENSURE  score $u_d$ and $x$
\STATE set $ i \leftarrow 1$ and $Q_i = 0$;
\STATE compute $x_i \leftarrow x_{i-1} + \delta_0 d$ and $\nabla f(x_i)$;
\REPEAT
\STATE set $s_Q \leftarrow Q_i$;
\STATE $i \leftarrow i +1$
\STATE compute $x_i \leftarrow x_{i-1} - \alpha\cdot 2a (x_{i-1} - x_0)$ and  $\nabla f(x_{i})$; \label{c1}
\STATE compute $ Q_i \leftarrow Q_{i-1} + \nabla f(x_{i})^\intercal(x_{i}-x_{i-1}) $ \label{c2}
\UNTIL{ $\{s_Q < 0\  \& \ Q_i  > 0 \}$ or $\|x_i - x_0\| \geq M$ }
\STATE compute $u_d= \max_{j=1,...,i}-\nabla f(x_j)^{\intercal} d$;
\STATE if $\|x_N - x_0\| \geq M$, set $u_d \leftarrow -|u_d|$;
\STATE \textbf{return} $u_d$ and $x = x_i$.
\end{algorithmic}
\end{algorithm}

Alg.~\ref{accelerate} aims to find an integer $i$ such that $Q_{i-1} < 0$ but $Q_{i} > 0$. The existence of such an $i$ can be illustrated as follows. It is seen that $Q_i = aF'(a)$ (please see Eq.~\ref{fa} in Appendix B). This implies that $Q_i=\int_{t_1}^{t_i} g'(t)dt$. When $\alpha\rightarrow 0$, we have $i_1$ and $i_2$ so that $|t_{i_1}-\xi|<\varepsilon$ and $|t_{i_2}-T|<\varepsilon$ for any $\varepsilon >0$, and $Q(i_1)<0$ and $Q(i_2)>0$. Thus, there exists an $i$ such that $Q(i)>0$ and $Q(i-1)<0$. 

Corollary~\ref{cor1} proves that if there exists a better local minimum $x'$ along $d$, then applying Alg.~\ref{alg:distance determination} or Alg.~\ref{accelerate}, we are able to escape from the local attraction basin of $x_0$.

\subsubsection{Policy on the sampling of promising directions}

In the following, we show how to sample directions that are of high probability to be promising. We first present a fixed policy, then propose to learn for an optimal policy by policy gradient.

\begin{algorithm}[htbp]\caption{Fixed policy on sampling promising direction}\label{alg:sampling with curiosity}
\begin{algorithmic}[1]
\REQUIRE a local minimizer $x_0$, an integer $P>0$ and $\sigma > 0$
\ENSURE a set of candidate directions and starting points
\STATE sample ${N_0}$ directions $\{d_i\}_{i=1}^{N_0}$ uniformly at random; apply Alg.~\ref{alg:distance determination} or Alg.~\ref{accelerate} to obtain their scores $\{u_i\}_{i=1}^{N_0}$ and  $\{x_i\}_{i=1}^{N_0}$; \label{fp1}
\STATE  set ${\cal S} \leftarrow \emptyset, {\cal D} \leftarrow \emptyset, t\leftarrow 1$. \label{fp6}
\REPEAT
\STATE sample
\[\tilde{d} = \sum_{i: u_i < 0} u_i d_i - \sum_{i: u_i > 0} u_i d_i +\varepsilon, \varepsilon\sim {\cal N}(0,\sigma^2)\]\label{fp2}
\STATE apply Alg.~\ref{alg:distance determination} or  Alg.~\ref{accelerate} to obtain $\tilde{u}$ and $\tilde{x}$.
\IF{$\tilde{u} > 0$}
	\STATE set ${\cal S}\leftarrow {\cal S}\bigcup \{\tilde{x}\}$ and  ${\cal D}\leftarrow {\cal D}\bigcup \{\tilde{d}\} $. \label{fp3}
\ENDIF
\STATE set  $u_i \leftarrow u_{i+1}, 1\leq i \leq N_0 -1 $ and $u_{N_0} \leftarrow \tilde{u}$; \label{fp4}
\STATE set  $d_i \leftarrow d_{i+1}, 1\leq i \leq N_0 -1 $ and $d_{N_0} \leftarrow \tilde{d}$; \label{fp7}
\STATE $t \leftarrow t + 1$.
\UNTIL{$t \geq P$.}
\STATE return  $\cal S$ and $\cal D$.
\end{algorithmic}
\end{algorithm}


Alg.~\ref{alg:sampling with curiosity} summarizes the fixed policy method. In Alg.~\ref{alg:sampling with curiosity}, first a set of directions are sampled uniformly at random (line~\ref{fp1}). Their scores are computed by Alg.~\ref{alg:distance determination} or Alg.~\ref{accelerate}. Archives used to store the directions and starting points are initialized (line~\ref{fp6}).  A direction is sampled by using a linear combination of previous directions with their respective scores as coefficients (line~\ref{fp2}). If the sampled direction has a positive score, its score and the obtained starting point are included in the archive. The sets of scores and directions are updated accordingly in a FIFO manner (lines~\ref{fp4}-\ref{fp7}). The algorithm terminates if the number of sampling exceeds $P$.


We hope that the developed sampling algorithm is more efficient than that of the random sampling in terms of finding promising direction. $P_r$ denotes the probability of finding a promising direction by using the random sampling, $P_c$ be the probability by the fixed policy. Then in Appendix C, we will do some explanation why $P_c>P_r$.



\subsubsection{The transition} In our MDP model, the probability transition $p(s_{t+1}|s_t,a_t)$ is deterministic. The determination of new starting point depends on the sampling of a new direction $d_t$ and its score $u_t$. New state $s_{t+1}$ is then updated in a FIFO manner. That is, at each time step, the first element $(\mathbf{u}_1,\mathbf{d}_1)$ in $s_t$ is replaced by the newly sampled $(\mathbf{u}_t, \mathbf{d}_t)$.

All the proofs in this section are given in Appendix B.



\subsection{Learning the Escaping Policy by Policy Gradient}\label{mechanism_RL}


In the presented policy, a linear combination of previous directions with their scores as coefficients is applied to sample a new direction. However, this policy is not necessarily optimal. In this section, we propose to learn an optimal policy by the policy gradient algorithm~\cite{Sutton1998Reinforcement}.  

The learning is based on the same foregoing MDP framework. The goal is to learn the optimal coefficients for combining previously sampled directions. We assume that at time $t$, the coefficients are obtained as follows:
\begin{eqnarray}
\mathbf{m}_t 	&=& g(-|\mathbf{u}_t|; {\theta});\nonumber\\
\mathbf{w}_t 	&=& -|\mathbf{u}_t| + \mathbf{m}_t;\nonumber
\end{eqnarray}where $\mathbf{u}_t = [u_1, \cdots, u_{N_0}]^\intercal$ and $\mathbf{d}_t = [d_1, \cdots, d_{N_0}]$. $\mathbf{m}_t \in \mathbb{R}^{N_0}$ is the output of a feed-forward neural network $g$ with parameter $\theta$, and $\mathbf{w}_t \in \mathbb{R}^{N_0}$ is the coefficients. The current state $s_t$ is the composition of $\mathbf{u}_t$ and $\mathbf{d}_t$.

Fig.~\ref{mechanism_RLfigure} shows the framework of estimating the coefficients and sampling a new direction at a certain time step. For the next time step, $\mathbf{u}_{t+1}$ and $\mathbf{d}_{t+1}$ are updated
\begin{eqnarray*}
&u_i = u _{i+1}, \text{for}\ 1\leq i \leq N_0 -1\  \text{and}\  u_{N_0} = \tilde{u}\\
&d_i = d _{i+1}, \text{for}\ 1\leq i \leq N_0 -1\  \text{and}\  d_{N_0} = \tilde{d}
\end{eqnarray*}and $\mathbf{u}_{t+1} = \{u_i\}, \mathbf{d}_{t+1} = \{d_i\}$, $s_{t+1}  = \{\mathbf{u}_{t+1}, \mathbf{d}_{t+1}\}$.

The policy gradient algorithm is used to learn $\theta$ for the neural network $g$. We assume that $\phi(s_t;\theta) = \mathbf{w}_t^\intercal \mathbf{d}_t $ and the policy can be stated as follows:
\begin{eqnarray}
\pi(d| s_t)  &=& {\cal N}\left(d| \phi(s_t;\theta), \sigma^2\right)\nonumber
\end{eqnarray}
The reward is defined to be
\begin{equation}\label{def_reward}
\begin{split}
\begin{array}{llll}
r_{1}		&=&\gamma \cdot \mathbb{I}(\tilde{u}>0) &\\
r_{t+1} 	&=&\gamma^t \cdot \mathbb{I}(\tilde{u}>0) & \text{if}\ r_{1:t-1}\ \text{are all zero}\\
r_{t+1} 	&=& 0					    &\text{if}\ r_{1:t}\ \text{are not all zero}
\end{array}
\end{split}
\end{equation}where $\gamma = 0.9$ is a constant, $\tilde{u}$ is the score of the sampled direction $\tilde{d}$ and $\mathbb{I}(\cdot)$ is the indicator function.


Alg.~\ref{alg:training RL} summaries the policy gradient learning procedure for $\theta$. $\theta$ is updated in $E$ epochs. At each epoch, first a sample of trajectories is obtained (lines~\ref{pgt1}-\ref{pgt2}). Given $x_0$, a trajectory can be sampled as follows. First, a set of $N_0$ initial directions is randomly generated and their scores are computed by Alg.~\ref{accelerate} (lines~\ref{pgt10}-\ref{pgt11}). $P$ new directions and their corresponding scores are then obtained (lines~\ref{pgt21}-\ref{pgt22}). At each step, the obtained direction $\tilde{d}$, the policy function $\phi(s_t;\theta)$ and the reward $r_{t+1}$ are gathered in the current trajectory $T_m$ (line~\ref{pgt23}). After the trajectory sampling, $\Delta\theta$ and $\theta_l$ are updated in lines~\ref{pg1}-~\ref{pg2} and line~\ref{pgo1}, respectively.
\begin{algorithm}[htbp]\caption{Training policy network with policy gradient}\label{alg:training RL}
\begin{algorithmic}[1]
\REQUIRE a local minimum $x_0$, an integer $P>0$, the number of training epochs $E$, the number of trajectories $N_T$, $\sigma > 0$ and learning rate $\beta>0$.
\ENSURE the optimal network parameter $\theta^*$.
\STATE randomly initialize $\theta_1 \in \mathbb{R}^d$; 
\FOR {$l=1:E$}
        \STATE {\em // create $N_T$ trajectories}  \label{pgt1}
        \FOR {$m=1:N_T$}
        \STATE set $T_m=\emptyset$;
        \STATE sample ${N_0}$ directions $\mathbf{d}_0 = \{d_i\}_{i=1}^{N_0}$ uniformly at random;\label{pgt10}
        \STATE apply Alg.~\ref{accelerate} to obtain their scores $\mathbf{u}_0 = \{u_i\}_{i=1}^{N_0}$ and  $\{x_i\}_{i=1}^{N_0}$; \label{pgt11}
        \STATE  set ${\cal S} \leftarrow \emptyset, {\cal D} \leftarrow \emptyset, t\leftarrow 0$;
    \REPEAT \label{pgt21}
	\STATE {\em // sample new direction}
        \STATE sample $\varepsilon\sim {\cal N}(0,\sigma^2) $;
        \STATE compute $\phi(s_t;\theta)= \sum [-|u_i| + g(-|\mathbf{u}_t|; {\theta_l})]_i d_i $ and $\tilde{d} =  \phi(s_t;\theta) +\varepsilon$;
        \STATE apply Alg.~\ref{accelerate} to obtain $\tilde{u}$;
        \STATE {\em // update the state}
        \STATE set $r_{t+1}$ by Eq.~\ref{def_reward}
        \STATE $u_i = u _{i+1}, 1\leq i \leq N_0 -1 $ and $u_{N_0} = \tilde{u}$;
        \STATE $d_i = d_{i+1}, 1\leq i \leq N_0 -1 $ and $d_{N_0} = \tilde{d}$;
        \STATE set $\mathbf{u}_{t+1} = \{u_i\}$, $\mathbf{d}_{t+1} = \{d_i\}$ and $s_{t+1}  = \{\mathbf{u}_{t+1}, \mathbf{d}_{t+1}\}$;
        \STATE {\em // update the trajectory}
        \STATE $T_m \leftarrow T_m\bigcup \{\tilde{d},\phi(s_t;\theta),r_{t+1}\}$; \label{pgt23}
        \STATE $t \leftarrow t + 1$.
\UNTIL{$t \geq P$.} \label{pgt22}
\ENDFOR \label{pgt2}
    \STATE {\em // policy gradient} \label{pg1}
    \STATE $\Delta\theta=0$
    \FOR {$m=1:N_T$}
    \STATE  $$\Delta\theta=\Delta\theta+
    \sum\limits_{s_t\in T_m}\nabla_{\theta}\phi(s_t;\theta)(\phi(s_t;\theta)-\tilde{d})\left(\sum\limits_{r_t\in T_m}r_t\right)$$
    \ENDFOR \label{pg2}
    \STATE $\theta_{l+1}=\theta_{l}+\beta\Delta\theta$; \label{pgo1}
\ENDFOR
\RETURN $\theta^* \leftarrow \theta_{E+1}$
\end{algorithmic}
\end{algorithm}

\begin{figure}
\centering\includegraphics[width = \columnwidth]{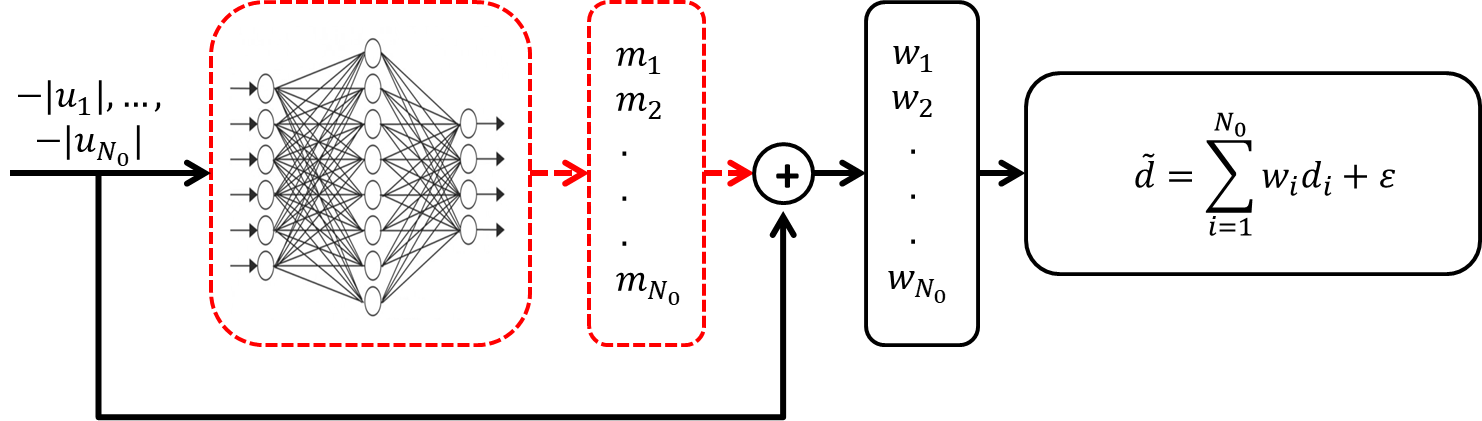}
\caption{The framework of the escaping policy on the $t$-th time step. The black solid line is the policy proposed in section~\ref{escaping_MDP}. The red dash line covers the feed-forward network and its output.}\label{mechanism_RLfigure}
\end{figure}

\subsection{Learning to be Global Optimizer}

Combining the proposed local search algorithm and the escaping policy, we can form a global optimization algorithm.  Alg.~\ref{adff} summarizes the algorithm, named as L$_2$GO. Starting from an initial point $x$, Gd-Net is firstly applied to obtain a local minimizer (line~\ref{lgo1}). The escaping policy is applied to sample $P$ new starting points (line~\ref{lgo2}). Gd-Net is then applied on these points (line~\ref{lgo3}). The algorithm terminates if the prefixed maximum number of escaping tries (i.e. $K$) has been reached (line~\ref{stop11}), or no new promising directions can be sampled (line~\ref{stop12}). 
If any of these conditions have been met, it is assumed that a global optimum has been found.

\begin{algorithm}\caption{The proposed global optimization algorithm based on learning to learn (L$_2$GO).}\label{adff}
\begin{algorithmic}[1]
\REQUIRE an initial point $x$, integers $K>0$
\ENSURE a global minimum $x^*$
\STATE $k\leftarrow 1$, stop = 0;
\STATE apply Gd-Net on $x$ to obtain a local minimizer $x_k$;\label{lgo1}
\REPEAT
\STATE if {$k \geq K$} stop = 1. \label{stop11}
\STATE apply the fixed escaping policy or the learned policy on $x_k$ to obtain a set $\cal S$ containing new starting points. \label{lgo2}
\IF {$\cal S=\emptyset$} \label{stop12}
\STATE $x^*\leftarrow x_k$, stop = 1.
\ELSE
\STATE apply Gd-Net to points in $\cal S$ to obtain a set of local minimizers, denoted as ${\cal S}^*$. \label{lgo3}
\STATE set $x_k=\argmin_{x\in {\cal S}^*} f(x) $
\STATE $k\leftarrow k+1$.
\ENDIF
\UNTIL{stop}
\STATE $x^*\leftarrow x_k$
\STATE return $x^*$.
\end{algorithmic}
\end{algorithm}

\begin{note}We should highlight that our method surpasses some filled function methods in the sense that our method has more chances to escape from local optimum. For example, consider the following filled function~\cite{Liang2007A}:
\begin{equation}\label{filled2}
H(x)=-a\|x-x_0\|^2+\min\{0,f(x)-f(x_0)\}^3
\end{equation}
\end{note}

The existence of the stable point $x_{\text{fill}}$ to $H$ usually holds. But $x_{\text{fill}}$ can be a saddle point or a local optimizer. If $x_{\text{fill}}$ is a saddle point, then to escape $x_0$, it is only possible by searching along $d_{\text{fill}}= x_{\text{fill}} - x_0$. However, it is highly unlikely $d_{\text{fill}}$ be contained in the pre-fixed direction set of the traditional filled function methods. This indicates that the corresponding filled function method will fail. 

On the other hand, if $x_{\text{fill}}$ is a local minimizer of $H(x)$, we can prove that the proposed policy can always find a promising solution. Theorem~\ref{surpass} summarizes the result. \newline
\begin{thm}Suppose there exists an attraction basin of $x_{\text{fill}}$ on the domain of $H(x)$, denoted as $B_{\text{fill}}$, then $\forall$ $x\in B_{\text{fill}} $, we have $u_d > 0$ for $d = x - x_0$.\label{surpass}
\end{thm}
\begin{proof}
We first prove that $\forall x \in B_{\text{fill}}, d=x-x_0$, $\exists  t\in \mathbb{R}$, s.t. $ f(x_0+t\cdot d)<f(x_0)$. This can be done by contradiction. If there is no such $t$, then \begin{equation}
\nabla H(x_0+t\cdot d)=-2at\cdot d.\label{ed}\end{equation}This is because that $\forall t \in \mathbb{R}_{++}$, we have $f(x_0+t\cdot d)\geq f(x_0)$, $H(x)$ degenerates to $-a\|t\cdot d\|^2$. Eq.~\ref{ed} implies that apply gradient descent from $x$ on $H(x) $ will not lead to a point in $B_{\text{fill}}$. This contradicts our assumption that $x \in B_{\text{fill}}$.

The existence of $t$ implies $u_{d} > 0$ by Theorem~\ref{reward}.
\end{proof}
\section{Experiment Results}\label{experiments}

In this section, we study the numerical performance of Gd-net, the escaping policies, and L$_2$GO.

\subsection{Model-driven Local Search}

This section investigates the performance of Gd-Net. In the experiments, 50 d-Net blocks are used. Parameters of these blocks are the same.

\textbf{Training.} d-Net is trained through minimizing the Monte Carlo approximation to the loss functions as defined in Eq.~\ref{loss}, in which a sample of the Gaussian family ${\cal F}_{\cal G}$ is used. In the experiments, we use ten 2-d Gaussian functions with positive covariance matrix as the training functions. d-Net is trained on 25 initial points sampled uniformly at random for each training function. At each layer of d-Net, the step size $\alpha_k$ is obtained by exact line search in [0,1]~\footnote{Note that taking $\alpha_k \in (0,1]$ is not necessarily the best choice for line search. It is rather considered as a rule of thumb. Notice that limiting the search of $\alpha_k$ in (0,1] could make Gd-Net be scale-variant. We transform $f(x)=f(x)/f(x_0)$ in order to eliminate the scaling problem where $x_0$ is the initial point when testing.}. Gradient descent is used to optimize Eq.~\ref{loss} with a learning rate 0.1 for 100 epochs. The same training configuration is used in the following. 

\textbf{Testing.} We use functions sampled from ${\cal F}_{\cal G}$ in 5-d, and $\chi^2$-functions\footnote{The $\chi^2$-function is of the following form
\begin{equation*}
f(x)=\frac{x^{k/2-1}e^{-x/2}}{2^{k/2}\Gamma(k/2)}, x>0
\end{equation*}where $\Gamma(\cdot)$ is the gamma function and $k$ is a parameter. } in 2-d to test Gd-Net. Note that Gd-Net is trained on 2-d ${\cal F}_{\cal G}$. By testing on 5-d Gaussian functions, we can see its generalization ability on higher-dimensional functions. The testing on $\chi^2$ functions can check the generalization ability of Gd-Net on functions with non-symmetric contour different to Gaussians. Fig.~\ref{contour} shows the difference between Gaussian and $\chi^2$ contour.
\begin{figure}
\centering\includegraphics[width = \columnwidth]{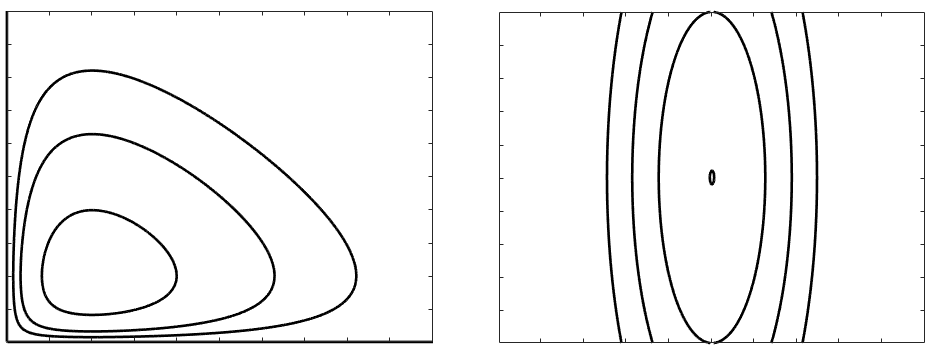}
\caption{A demo on the difference between a Gaussian contour and a $\chi^2$ contour.}\label{contour}
\end{figure}

Fig.~\ref{highdim} shows the testing result of the learned Gd-Net on optimizing a 5-d Gaussian function with different initial points. The test on a $\chi^2$ function is shown in Fig.~\ref{chi}. In these figures, first-order and second-order optimization algorithms, including steepest gradient descent, conjugate descent and BFGS, are used for comparison. From the figures, it is clear that Gd-Net requires much less iterations to reach the minimum than the compared algorithms.  


Further, we observed that unlike BFGS, where a positive-definite Hessian matrix is a must, Gd-Net can cope with ill-conditioned Hessians. Fig.~\ref{mutiblock} shows the results on a 2-d Gaussian function with ill-posed Hessian. For an initial point that is far away from a minimizer, its Hessian is nearly singular which implies that the search area is rather flat. From the left plot of Fig.~\ref{mutiblock}, it is seen that Gd-Net gradually decreases, while the other methods fail to make any progress. On the right plot, it is seen that Gd-Net finally progresses out the flat area and the criterion starts decreasing quickly.

\begin{figure}
\centering\includegraphics[width = \columnwidth]{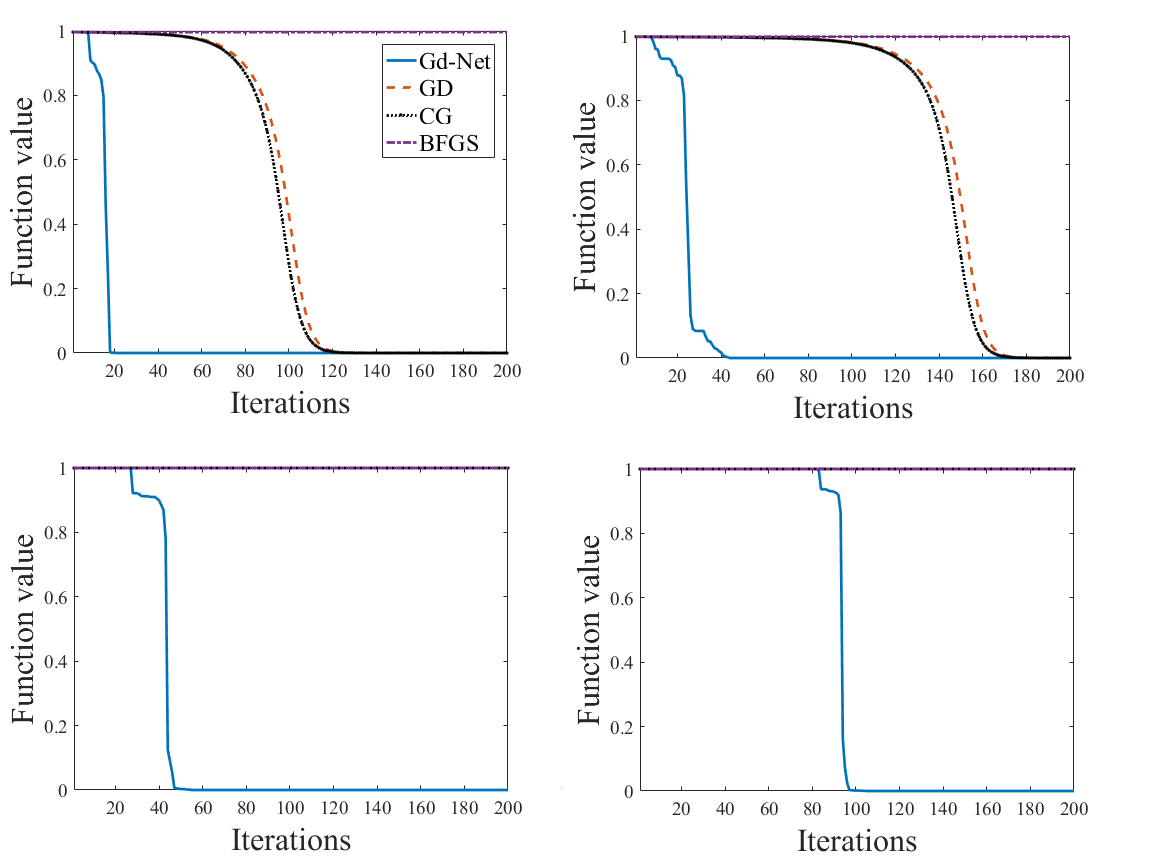}
\caption{The optimization curve of the learned Gd-Net on a 5-d Gaussian function with various initial points.}\label{highdim}
\end{figure}

\begin{figure}
\centering\includegraphics[width = \columnwidth]{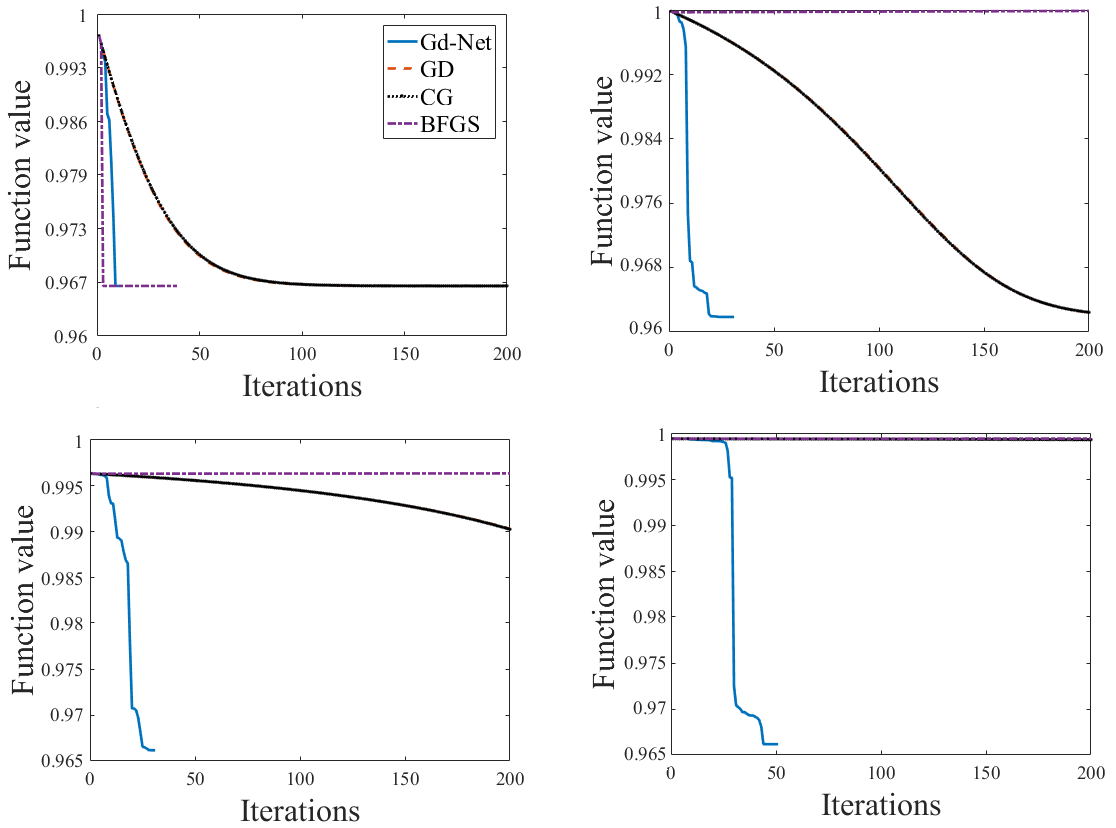}
\caption{The optimization curve of the learned Gd-Net on a 2-d $\chi^2$ function with two different initial points.}\label{chi}
\end{figure}

\begin{figure}
\centering
\includegraphics[width = \columnwidth]{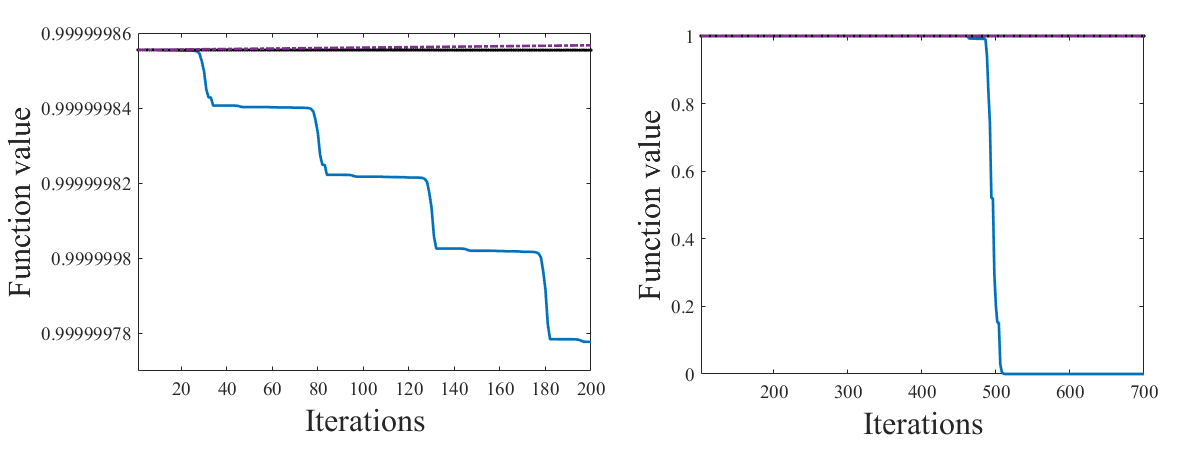}
\caption{The optimization procedure of Gd-Net (with 5 blocks) on a 5-d Gaussian function with an initial point far away from the optimum. The left plot shows the decreasing curve obtained by the first 4 blocks, while the right shows the curve of the rest block.}\label{mutiblock}
\end{figure}

\subsection{The fixed escaping policy}

In this section, controlled experiments are carried out to justify the ability of the fixed policy. We first consider a low-dimensional non-convex optimization problem with two local minimizers, then a high-dimensional highly non-convex problems with many local minimizers. The fixed policy is compared with random sampling on these test problems.

\subsubsection{Mixture of Gaussians functions}

Consider the following mixture of Gaussians functions
\begin{eqnarray}\label{mix_gauss}
f(x)&=&- \sum_{i=1}^m c_i\exp\left\{-(x-\mu_i)^\intercal{\Sigma_i}^{-1}(x-\mu_i)\right\}
\end{eqnarray}where $x \in \mathbb{R}^n, c_i > 0$ and $\Sigma_i \succeq 0$. The mixture of Gaussian functions have $m$ local minimizers at $\mu_i$'s. 


In the experiments, we set $m=2,3$. To test the ability of the escaping scheme, we assume the escaping starts from a local minimizer. We test on dimension $n=2, 3, 5, 8, 10$. The other algorithmic parameters are $N_0 = 2, 3, 5, 8, 10$, $P=15,20,50,100,250$ for $n=2, 3, 5, 8, 10,$ respectively, and $\sigma=0.1,\delta_0=0.2,N=20$. 

Table~\ref{sampleno} shows the average number of samplings used to escape from local optimum and the standard deviation (in brackets) over 500 runs obtained by using the fixed policy and the random sampling method. 

From the table, it is observed that the fixed policy requires less samples than that of the random sampling, and the standard deviation is smaller. The $p$-value obtained by applying the rank sum hypothesis test at $5\%$ significance level is shown in the last column. The hypothesis test suggests that the fixed policy outperforms the random sampling approach significantly (the p-value is less than $0.05$).
\begin{table}[h]
\caption{The number of samplings used to escape.}\label{sampleno}
\centering
\begin{tabular}{llllr}
\toprule
$n$ &	$m$ & random search & fixed policy & $p$-value \\
\midrule
2 & 2& $4.25(3.51)$ & $\textbf{3.03(2.96)}$ & 0.00003\\
3 &2& $6.22(5.42)$ & $\textbf{4.75(4.55)}$ & 0.00003\\
5 & 2&$12.38(13.07)$ & $\textbf{9.18(8.27)}$ & 0.00002\\
8 & 2&$45.6(31.46)$ & $\textbf{39.5(30.12)}$ & 0.005\\
10& 2&$110.65(78.89)$ & $\textbf{99.6(74.78)}$ & 0.032\\
2 & 3&	$6.72(2.12)$ & $\textbf{5.53(0.96)}$ & 0.00003\\
\bottomrule
\end{tabular}
\end{table}

\subsubsection{Deep neural network}

The loss function of a deep neural network has many local optimizers. We take the training of a deep neural network for image classification on CIFAR-10 as an example. For CIFAR-10, an 8-layer convolution neural network similar to Le-Net~\cite{lecun1998gradient}, with 2520-d parameters, is applied. The cross entropy is used as the loss function.

The number of local minimizers found by a method is used as the metric of comparison. Given a maximal number of attempts $P$,  a larger number of local minimizers indicates a higher probability of escaping local minimum, and hence a better performance. For CIFAR-10, ADAM~\cite{Kingma2014Adam} with mini-batch stochastic gradient is applied in the minimization phase. 




Note that existing filled functions often involve $f(x)$. This usually makes the application of mini-batch stochastic gradient method difficult if $f(x)$ is not sum of sub-functions. Instead, the auxiliary function $\widetilde{H}(x)$ used in this paper does not involve $f(x)$. Fig.~\ref{minibatch} shows the scores (cf. Eq.~\ref{scores}) against the distance to current local optimum with different mini-batch sizes. From the figure, it is seen that with different batch-size, the scores exhibit similar behavior. This shows the applicability of the proposed escaping method to stochastic-based local search algorithms. In the experiment, the parameters to apply Alg.~\ref{alg:distance determination} is set as $N_0=300$, $P=1000,N=10,\sigma=0.01, a=1$ and $\delta_0=0.5$.

In the following, the effective samplings\footnote{A sampling is effective if the sampled direction is with positive score.} in 1000 samplings are used to compare the proposed escaping policy against the random sampling. The obtained promising directions with different thresholds in 500 runs are summarized in Table~\ref{promising direction CNN}. It is clear that the proposed escaping policy is able to find more samples with positive scores than that of the random sampling.
\begin{table}[h]
\caption{The number of effective samplings in 1000 samplings.}\label{promising direction CNN}
\centering
\begin{tabular}{llllr}
\toprule
 & score$>$0 & score$>$0.01 & score$>$0.03 & score$>$0.05\\
\midrule
random sampling & 52 &  19  & 2 & 0\\
fixed policy &   423 &  201 & 31& 8\\
\bottomrule
\end{tabular}
\end{table}
\begin{figure}
\includegraphics[width = 0.8\columnwidth]{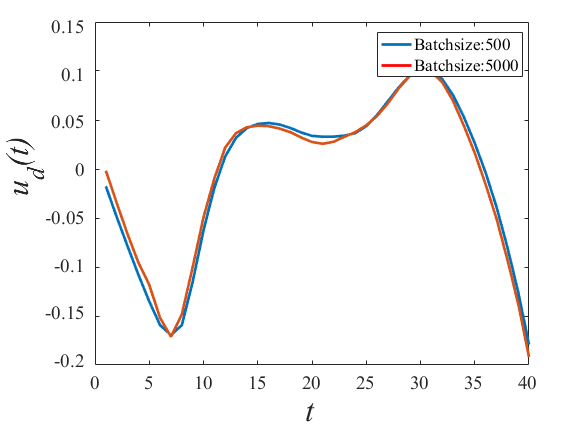}\centering
\caption{The curves of the score against the step size w.r.t. two mini-batches when training a convolution neural network for CIFAR-10.}\label{minibatch}
\end{figure}

\subsection{The fixed policy vs. the learned policy}

In this section we show the effectiveness of the learned policy. The policy function $\phi(s_t;\theta)$ in Eq.~\ref{policy_form} is set as a 3 layer network with sigmoid as hidden layer activation function, and a fully-connected output layer with linear activation function.

\textbf{Training.} A single Gaussian mixture function with two local minimizers is used for training the policy network in 2-d and 5-d, respectively. The other parameters are set $N_0=2 (5);P=15 (50); N_T=20 (50)$ and $E=30$ for 2 (5)-d. The number of hidden layer units is 5 and 200 for 2-d and 5-d, respectively. 

\textbf{Testing.} To test the learned policy, we also use the Gaussian mixture functions with two local minimizers in 2-d and 5-d. Here we set $m=2, c_1=c_2=1$. Table~\ref{rl_sample} shows the average number of samplings used to escape from local optimum and the standard deviation (in brackets) over 500 runs obtained by using the fixed policy, the learned policy and the random sampling. Detailed configurations of the functions used for training can be found in Appendix D.

\begin{table}[h]
\caption{The average number of samplings used to find the promising direction in different settings. }\label{rl_sample}
\centering
\begin{tabular}{llllr}
\toprule
 & n & random search & learned policy & fixed policy \\
\midrule
1 & 2	& 	$4.25(3.51)$ & $\textbf{2.67(1.77)}$ & 3.03(2.96)\\
2 & 2 	& 	$4.81(4)$ & $\textbf{3.07(2.05)}$ & 3.52(13.27)\\
3 & 2	&	$7.03(5.45)$ & $\textbf{4.12(2.63)}$ & 5.5(4.87)\\
4 & 2	& 	$5.69(4.51)$ & $\textbf{3.37(2.23)}$ & 5.05(4.50)\\
5 & 5	&	$12.76(10.77)$ & $\textbf{5.78(5.21)}$ & 10.22(9.35)\\
6 & 5	&	$25.32(16.21)$ & $\textbf{13.08(11.96)}$ & 21.65(15.81)\\
7 & 5	&	$16.16(13.15)$ & $\textbf{8.11(7.55)}$ & 11.68(10.61)\\
8 & 5	&	$7.25(6.30)$ & $\textbf{4.01(4.32)}$ & 4.39(4.38)\\
9 & 5	&	$11(10.29)$ & $\textbf{7.28(7.69)}$ & 8.48(7.95)\\
\bottomrule
\end{tabular}
\end{table}

From the table, it is seen clearly that the learned policy requires less samplings to reach new local optimum. To observe the behaviors of the compared escaping policies better, Fig.~\ref{percentage} shows the histograms of the number of effective samplings for a 2-D function. From the figure, we see that the fixed policy is mostly likely to escape the current local optimum in one sampling, but it also is highly possible to require more samplings. That is, the number of effective samplings by the fixed policy follows a heavy-tail distribution. For the learned policy, the effective sampling numbers are mostly concentrated in the first 7 samplings. This shows that the learned policy is more robust than the other policies, which can also be confirmed in Table~\ref{rl_sample} by the standard deviations. We may thus conclude that the learned policy is more efficient than the fixed policy and random sampling.

\begin{figure}
\centering\includegraphics[width = 0.9\columnwidth]{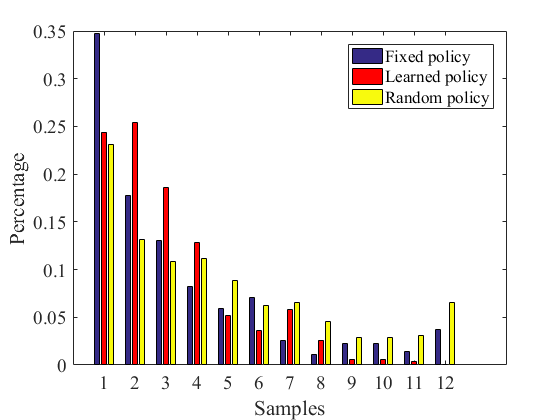}\centering
\caption{The histogram of the number of effective samplings for a 2-D Gaussian mixture function by the fixed policy, the learned policy and the random sampling policy. }\label{percentage}
\end{figure}

\subsection{The global search ability}

In this section, we study the global search ability of L$_2$GO in comparison with the filled function method proposed in~\cite{Ge1987A}. Three examples, including the three hump function, robust regression and neural network classifier, are used as benchmarks.

\subsubsection{Three-hump function}

The function is defined as \[f(x)=2x_1^2-1.05x_1^4+x_1^6/6-x_1 x_2+x_2^2.\]It has three local minimizers at $[-1.73,-0.87]^\intercal$, $[0,0]^\intercal$, and $[1.73,0.87]^\intercal$. The global minimizer is at $[0,0]^\intercal$. In our test, the algorithm parameters are set as $N_0 = 2, P=15,\sigma=0.1,\delta=0.2$ and $N=20$. We run L$_2$GO 20 times with different initial points. Fig.~\ref{three_hump} shows the averaged optimization process of L$_2$GO. From the figure, it is seen that L$_2$GO is able to reach the local minimizer one by one. It is also seen that Gd-Net performs better than BFGS and steepest descent. 

\begin{figure}
\centering\includegraphics[width = 0.9\columnwidth]{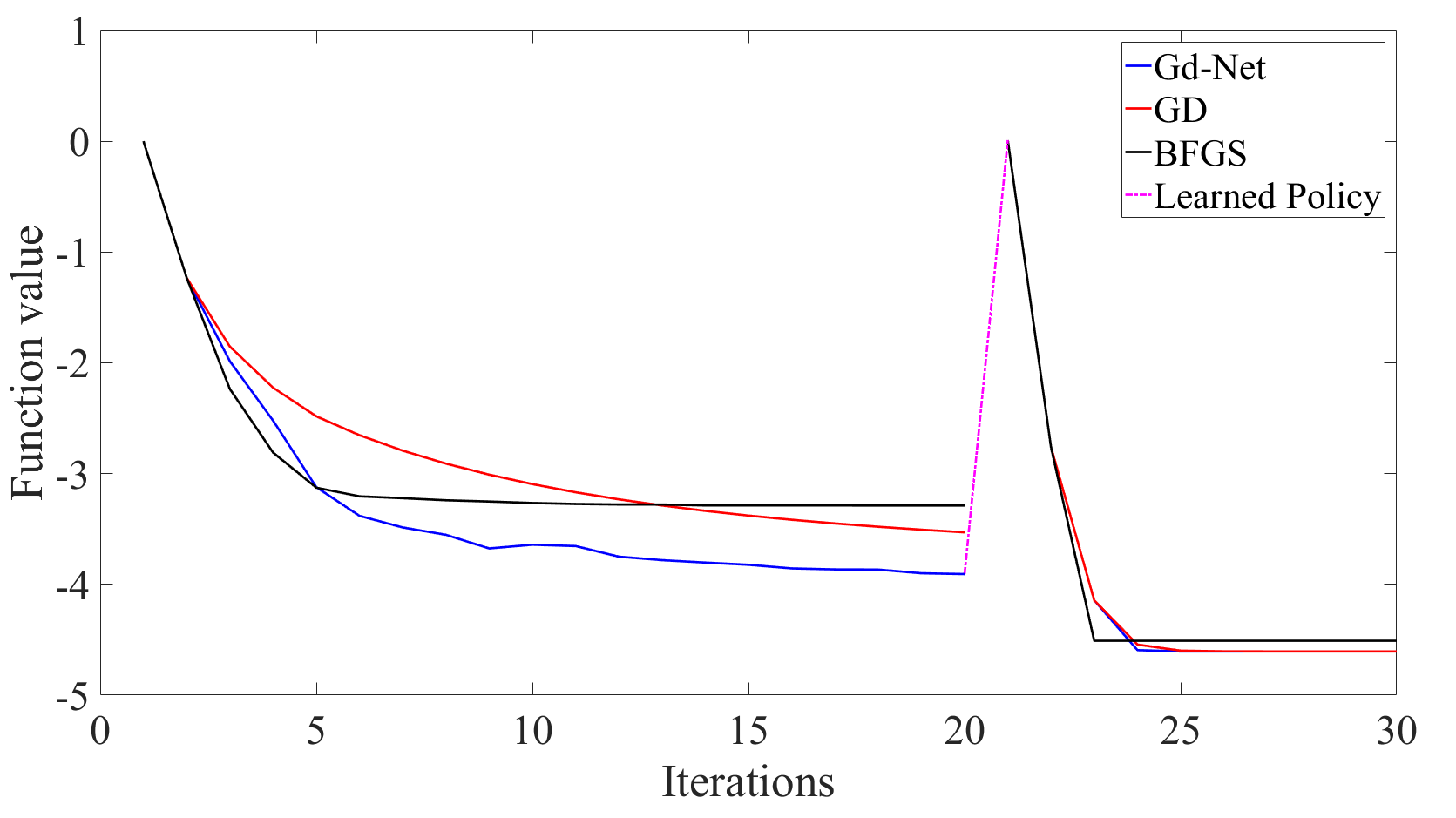}
\caption{The running procedure of L$_2$GO, in which GD, BFGS and Gd-Net are used as local search, while the learned policy is used to escape from local minimum (represented in pink dotted line). }\label{three_hump}
\end{figure}

Table~\ref{sample_three} shows the number of effective samplings when using the fixed policy, random sampling and learned policy as escaping scheme. It is seen that the filled function method has failed due to the existence of the saddle point as shown in Fig.~\ref{filled}.
\begin{table}[h]
\caption{The number of samplings used to find the promising directions for the three-hump function.}\label{sample_three}
\centering
\begin{tabular}{lllr}
\toprule
 random sampling & learned policy & fixed policy & filled function\\
\midrule
$5.35(4.38)$ & $\textbf{3.29(2.07)}$ & 3.76(3.26) & NA\\
\bottomrule
\end{tabular}
\end{table}


\subsubsection{Robust regression}

For the robust linear regression problem~\cite{li2016learning}, a popular choice of the loss function is the Geman-McClure estimator, which can be written as follows:
\begin{equation}\label{1}
\mathbf{P}:\min\limits_{w,b}f(w,b)=  \frac{1}{n}\sum_{i=1}^n\frac{(y_i-w^\intercal x_i-b)^2}{(y_i-w^\intercal x_i-b)^2+c^2}
\end{equation}where $w\in \mathbb{R}^d, b\in \mathbb{R}$ represent the weights and biases, respectively. $x_i\in \mathbb{R}^d$ and $y_i\in \mathbb{R}$ is the feature vector and label of the $i$-th instance and $c\in \mathbb{R}$ is a constant that modulates the shape of the loss function.

The landscape of the robust regression problem can be systematically controlled. Specifically, we can decide the number of local minimizers, their locations and criteria readily. Note that given $\{w, b\}$, the training data can be created by \begin{equation}\label{robust}y_i = w^\intercal x_i + b_i + \epsilon\end{equation} Different $\{w,b\}$ indicates different local minimum.

In our experiments, we randomly sample 50 points of $x_j \sim {\cal N}(0, \mathbb{I})$ in $\mathbb{R}^2$, and divide them to two sets $\mathcal{S}_1 = \{x_j, 1\leq i\leq 10\}$ and $\mathcal{S}_2 = \{x_j, 11\leq i\leq 50\}$. For each set $\mathcal{S}_i, 1\leq i\leq 2$, give a $\{w_i, b_i\}$, apply $y_j = w_i^\intercal x_j + b_i + \epsilon$, a training set ${\cal T}_i$ can be obtained. Combining them, we obtain the whole data set ${\cal T} = \cup {\cal T}_i$. Given this training set, it is known that the objective function has two obvious local minimizers at $(w_1,b_1)$ and $(w_2,b_2)$ and lots of other local minimizers. Please see Fig.~\ref{ro_contour} for contour of the robust regression function with $w_1 = (-8, -8)$, $w_2 = (5,5)$ and $b_1=b_2 = 0$. There are two main local minimizers at $(w_1,b_1)$ and $(w_2,b_2)$ with $f(w_2, b_2) < f(w_1, b_1)$), and many other local minimizers.

\begin{figure}
\centering\includegraphics[width = 0.9\columnwidth]{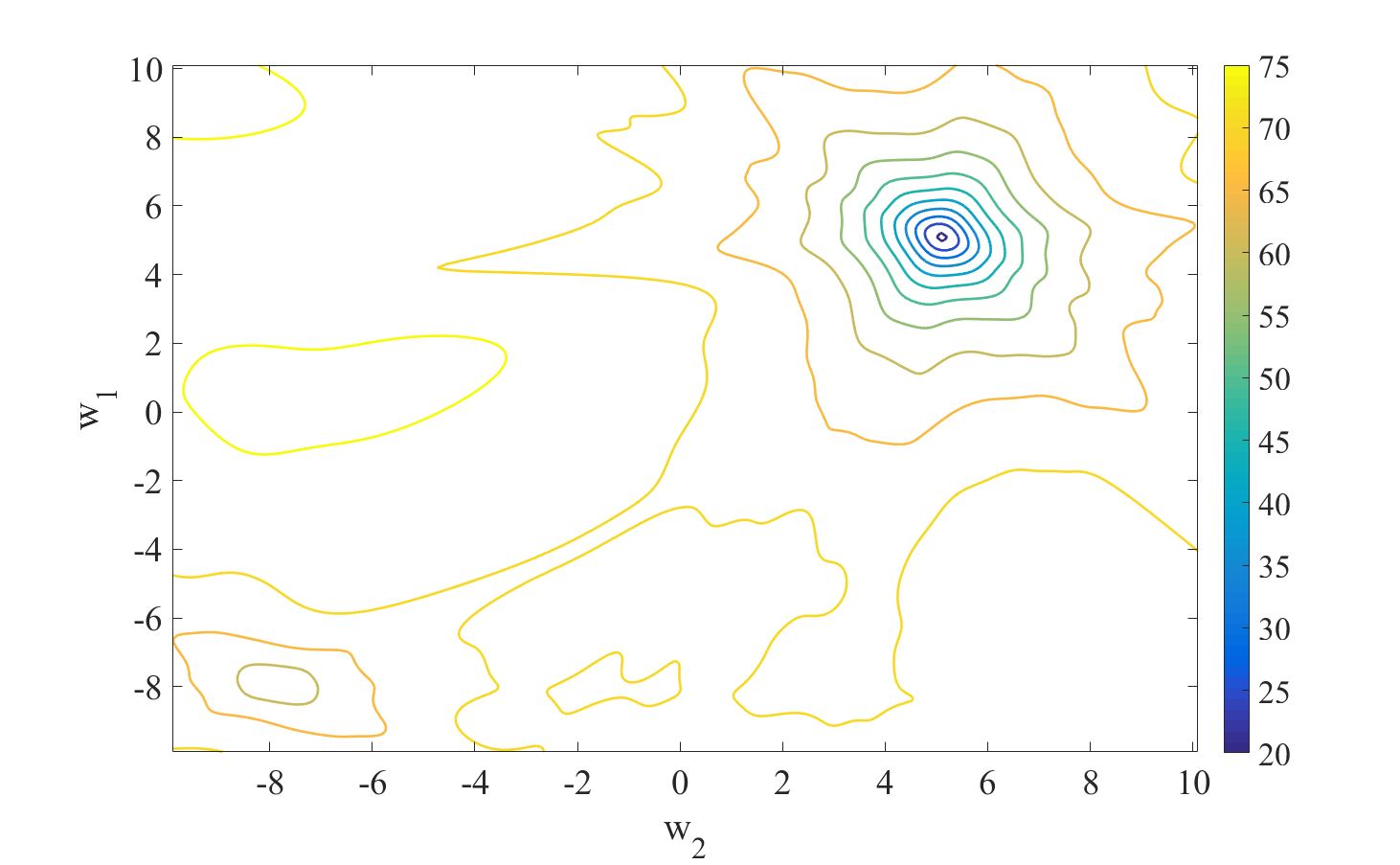}
\caption{The contour of the robust regression function with $w_1=(-8,-8), w_2 = (5,5)$ and $b_1=b_2 = 0$, respectively. }\label{ro_contour}
\end{figure}

Fig.~\ref{ro_reg} shows the optimization curve of the robust regression function, in which Gd-Net and GD are compared. We notice that BFGS is not convergent in this case since landscape here is vary flat. From the figure, we can see that for robust regression function, Gd-Net also performs better than GD. Table~\ref{sample_ro} shows the average numbers of effective samplings obtained by the compared policies. From the table it is clear that the learned policy performs the best, while the filled function method needs much more times. 

\begin{figure}
\includegraphics[width = 0.9\columnwidth]{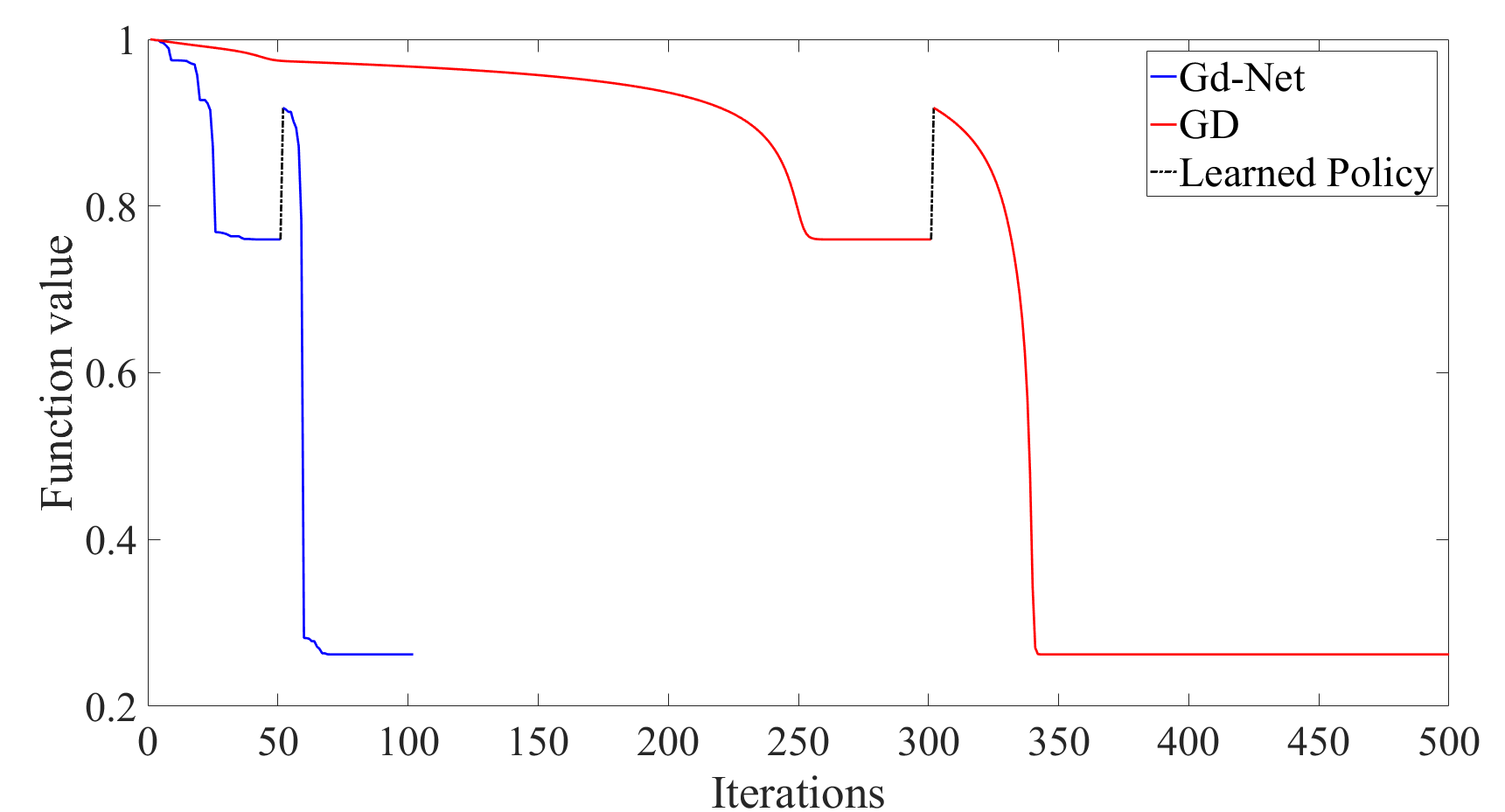}\centering
\caption{The optimization procedure of L$_2$GO and the steepest gradient with the learned policy on robust regression.}\label{ro_reg}
\end{figure}

\begin{table}[h]
\caption{The number of samplings used to find the promising directions for robust regression function.}\label{sample_ro}
\centering
\begin{tabular}{lllr}
\toprule
 random sampling & fixed policy & learned policy & filled function\\
\midrule
$13.05(9.48)$ & $\textbf{11.3(8.83)}$ & 12.23(9.27) & 89(20.22) \\
\bottomrule
\end{tabular}

\end{table}



\subsubsection{neural network classifier}

We construct a small network with one hidden layer for a classification problem in 2-d~\cite{li2016learning}. The number of hidden layer is one, and the total dimension of network is 5. The goal is to classify $\mathcal{S}_i=\{x_i+\varepsilon|\varepsilon \sim \mathcal{N}(0,\sigma^2)\}, i=1,2$ into two classes, where $x_1 \neq x_2 \in \mathbb{R}^2$. The cross entropy is used as the loss function. ADAM is compared with L$_2$GO. In ADAM, the learning rate is 0.001, and the hyper-parameters for momentum estimation are $0.9$ and $0.9$.

Fig.~\ref{class_small} shows the optimization curve. From the figure, we see that L$_2$GO performs better than ADAM. It can escape from the local optimum, and reach a better optimum successfully.
\begin{figure}
\includegraphics[width = 0.9\columnwidth]{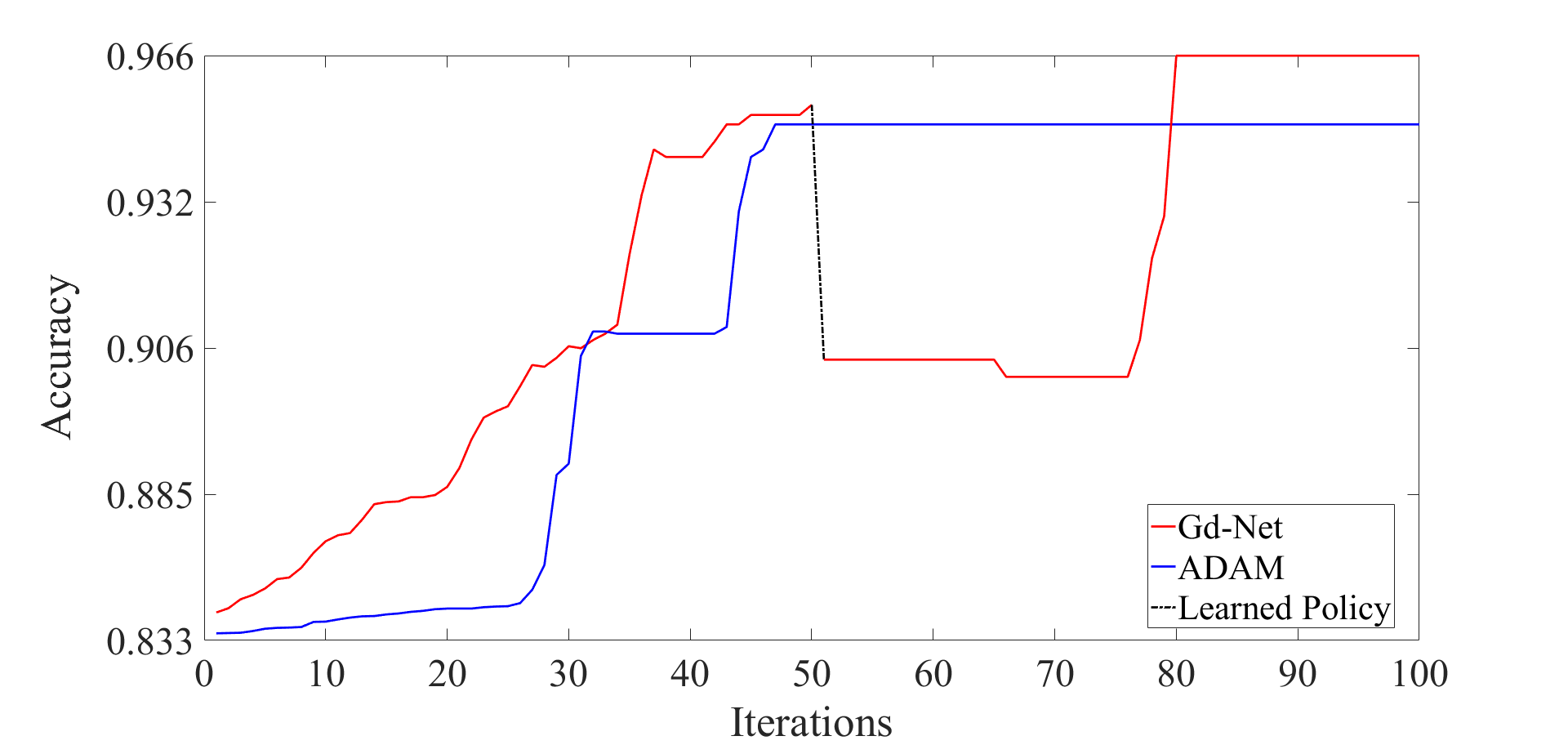}\centering
\caption{The optimization curve of L$_2$GO and ADAM on training the classification network. }\label{class_small}
\end{figure}

\section{Conclusion and Future Work}\label{discussion}

This paper proposed a two-phase global optimization algorithm for smooth non-convex function. In the minimization phase, a local optimization algorithm, called Gd-Net, was obtained by the model-driven learning approach. The method was established by learning the parameters of a non-linear combination of different descent directions through deep neural network training on a class of Gaussian family function. In the escaping phase, a fixed escaping policy was first developed based on the modeling of the escaping phase as an MDP. We further proposed to learn the escaping policy by policy gradient.

A series of experiments have been carried out. First, controlled experimental results showed that Gd-Net performs better than classical algorithms such as steepest gradient descent, conjugate descent and BFGS on locally convex functions. The generalization ability of the learned algorithm was also verified on higher dimensional functions and on functions with contour different to the Gaussian family function. Second, experimental results showed that the fixed policy was more able to find promising solutions than random sampling, while the learned policy performed better than the fixed policy. Third, the proposed two-phase global algorithm, L$_2$GO, showed its effectiveness on a benchmark function and two machine learning problems.

In the future,  we plan to work on the following avenues. First, since the Hessian matrix is used in Gd-Net, it is thus not readily applicable to high-dimensional functions. Research on learning to learn approach for high-dimensional functions is appealing. Second, we found that learning the escaping policy is particularly difficult for high-dimensional functions. It is thus necessary to develop a better learning approach. Third, the two-phase approach is not the only way for global optimization. We intend to develop learning to learn approaches based on other global methods, such as branch and bound~\cite{Lawler1966Branch}, and for other types of optimization problems such as non-smooth, non-convex and non-derivative.

\bibliographystyle{IEEEtran}
\bibliography{ref}

\section*{Appendix A}\label{proof_AGD}

The proof of Theorem~\ref{dNetConv} can be found below.
\begin{proof}
In the AGD, $d_k=-R_{k-1} \widetilde{H}_{k-1} g_k$ where \[R_{k-1}=I-\frac{s_{k-1} \widehat{y}_{k-1}^\intercal}{s_{k-1}^\intercal \widetilde{y}_{k-1}},\]$ \widetilde{y}_{k-1}={w_k^3{g_k}-w_k^4{g_{k-1}}},\widehat{y}_{k-1}={w_k^1{g_k}-w_k^2{g_{k-1}}}$, and $\widetilde{H}_{k-1}={\beta}_k{H_{k-1}}+(1-{\beta}_k)I$.
With exact linear search, we have $g_k^\intercal s_{k-1} = 0$, therefore
\begin{eqnarray*}
-g_k^\intercal d_k&=&g_k^\intercal({\beta}_k{H_{k-1}}+(1-{\beta}_k)I)g_k\\
&-&g_k^\intercal \frac{s_{k-1} \widehat{y}_{k-1}^\intercal}{s_{k-1}^\intercal \widetilde{y}_{k-1}}\widetilde{H}_{k-1}g_k\\
&=&g_k^\intercal({\beta}_k{H_{k-1}}+(1-{\beta}_k)I)g_k
\end{eqnarray*}
It is clear that $-g_k^\intercal d_k>0$ if $H_{k-1} \succ 0$, otherwise a $\beta_k > 0$ can be chosen to make $({\beta}_k{H_{k-1}}+(1-{\beta}_k)I)$ diagonally dominant, which means $g_k^\intercal({\beta}_k{H_{k-1}}+(1-{\beta}_k)I)g_k>0$. Therefore, we can always make sure $g_k^\intercal d_k < 0$, i.e. $d_k$ is a descent direction, and
\[f(x_1) \geq f(x_2) \geq \cdots \geq f(x_k)\]Since $f(x)$ is bounded, there exists $f(x^*)$ such that $\lim_{k\rightarrow \infty }f(x_k) = f(x^\star)$.
\end{proof}

\section*{Appendix B}\label{proof_score}

This section gives details of the proofs for the theorems in Section~\ref{l2local}. The proof to Theorem~\ref{reward} is shown below.
\begin{proof}
According to assumption (2), $g(t)$ is convex in $[-\delta,\delta]$. As $g'(0)= f'(x)\big|_{x=x_0}=0,g''(0)>0$, then $g'(t)>0$ in $(0,\delta]$. Therefore, $g(t)$ is monotonically increasing in $[0,\delta]$, and monotonically decreasing in $[T-\delta, T]$. Since $f(x)\in C^2(\mathbb{R}^n)$, then $g(t)\in C^2[0,T]$. This implies that $g'(t)$ is continuous in $[0,T]$.

Since $g'(\delta)>0$ and $g'(T-\delta)<0$, then there exists a $\xi$ such that $g'(\xi)=0$. Further, $\xi$ is unique since it is assumed that there is no other local minima between $x_0$ and $x_1$. Then we have $g'(t)>0$ in $[0,\xi)$, and $g'(t)<0$ in $(\xi,T]$.
\end{proof}

To prove Theorem~\ref{existence}, we first prove Lemma~\ref{lemma1}.

\begin{lemma}\label{lemma1}
Suppose $F(a)$ to be the function defined in Alg.~\ref{alg:distance determination}. For fixed $a$ and $N$, we have:
\begin{equation}
  \lim_{\alpha\rightarrow 0} F'(a) =\frac{1}{a}\int_{t_1}^{t_N} g'(t)dt
\end{equation}where $g(t)=f(x_0+td)$.
\end{lemma}
\begin{proof}
Since $\{x_1, \cdots, x_N\}$ are all on the line  $x_0+td$ with $t > 0$. Therefore, each $x_i, 1\leq i\leq N$ can be written as
\[x_i = x_0 + t_i d\]where $t_1 = \delta_0 >0$ and $t_1 < t_2 \cdots < t_N$.
Further, we have
\begin{eqnarray*}
 x_i - x_{i-1} = 2a\alpha t_{i-1} d  \Leftrightarrow (t_i - t_{i-1}) \cdot d = 2a\alpha t_{i-1}  d
\end{eqnarray*}or equivalently,
\begin{eqnarray}
t_i - t_{i-1} = 2a\alpha t_{i-1} \Longrightarrow t_i = (1+2a\alpha)^{i-1} t_1
\end{eqnarray}
We thus have:
\begin{align}
F'(a)	&=	\sum_{i=1}^N \nabla_a f(x_i)/(i-1)	\nonumber\\
&=	\sum_{i=2}^N{\nabla_a f((1+2a\alpha)^{i-1} t_1d)}/(i-1)\nonumber \\
&= \sum_{i=2}^N  \nabla f((1+2a\alpha)^{i-1} t_1d)^\intercal \cdot 2\alpha t_{i-1}d \nonumber\\
&=\sum_{i=2}^N  \nabla f(x_{i-1}+a \cdot 2\alpha t_{i-1} d)^\intercal \cdot  \frac{1}{a}(t_i-t_{i-1})d \nonumber\\
&=\frac{1}{a}\sum_{i=2}^N  \nabla f(x_i)^\intercal\cdot d\cdot(t_i-t_{i-1})\label{fa}\\
&=\frac{1}{a}\sum_{i=2}^N  g'(t_i)(t_i-t_{i-1})\nonumber
\end{align}
 Since $g'(t)$ is continuous in $[t_1, t_N]$, it is Riemann integrable. Thus, for a fixed $N$, we have
 \begin{eqnarray*}
  \lim_{\alpha\rightarrow 0}\sum_{i=2}^N  g'(t_i)(t_i-t_{i-1})=\int_{t_1}^{t_N} g'(t)dt
\end{eqnarray*}This finishes the proof.
\end{proof}
In the sequel, we define \[G(a) = t_N - t_1 = ((1+2\alpha a)^{N-1} -1)t_1\]where $a \in [0, \infty)$. Here $G(a)$ is just the distance between $x_N$ and $x_1$ along $d$. Obviously, $G(a)$ is a polynomial function of $a$, and it is monotonically increasing.

Based on Lemma~\ref{lemma1}, Lemma~\ref{lth3} can be established.
\begin{lemma}
Suppose that $x' =x_0+Td$ is a point such that $f(x_0)\geq f(x')$, and there are no other local minimizer within ${\cal B}_0$. If $\alpha$ is sufficiently small, then there exists an $a^*$ such that $F'(a^*)=0$.\label{lth3}
\end{lemma}
\begin{proof}
Since $f(x') \leq f(x_0)$, and $g(t) = f(x_0 + td)$ is smooth, there exists a $\xi$ such that $\xi = \argmax_{t \in [t_1, T]} g(t)$.

Let's consider two cases. First, let $D_1 = \xi-t_1$, then there is an $a_1$ s.t. $G(a_1) = D_1$ according to the intermediate value theorem. As $\alpha$ is sufficiently small, we have: $\forall \varepsilon>0,\ \exists\ \alpha$ s.t.
\begin{eqnarray*}
\left|F'(a_1)-\frac{1}{a_1}\int_{t_1}^{t_N} g'(t)dt\right |<\varepsilon
\end{eqnarray*}
Note that we can choose $\delta_0$ s.t.\begin{eqnarray*}
\frac{1}{a_1}\int_{t_1}^{t_N} g'(t)dt=\frac{1}{a_1}(g(\xi)-g(t_1))&=&\\
\frac{1}{a_1}(f(x_0+\xi d)-f(x_0+\delta_0d))&>&0\end{eqnarray*}Let $\epsilon_0= \frac{1}{a_1}\int_{t_1}^{t_N} g'(t)dt$ and $\epsilon = \epsilon_0/2$, then $\exists\ \alpha_0$ s.t. $|F'(a_1)-\epsilon_0|<\epsilon_0/2\ \Rightarrow\ F'(a_1)>\epsilon_0/2>0$.

Similarly, if let $D_2 = T- t_1$, then there is $a_2$ s.t. $G(a_2) = D_2$, as $\alpha$ is sufficiently small, we have: $\forall \varepsilon>0,\ \exists\ \alpha$ s.t.
\begin{eqnarray*}
\left|F'(a_2)-\frac{1}{a_2}\int_{t_1}^{t_N} g'(t)dt\right |<\varepsilon
\end{eqnarray*}
Note that \begin{eqnarray*}
\frac{1}{a_2}\int_{t_1}^{t_N} g'(t)dt=\frac{1}{a_2}(g(T)-g(t_1))&=& \nonumber\\
\frac{1}{a_2} [f(x') - f(x_0 + \delta_0 d)] = \frac{1}{a_2} (f(x') - f(x_1))&<& 0\end{eqnarray*}Let $\epsilon_1= \frac{1}{a_2}\int_{t_1}^{t_N} g'(t)dt$ and $\epsilon = -\epsilon_1/2$, then $\exists\ \alpha_0$ s.t. $|F'(a_2)-\epsilon_1|<-\epsilon_1/2\ \Rightarrow\ F'(a_2)<\epsilon_1/2<0$.

In summary, we have $F'(a_1) > 0$ and $F'(a_2) < 0$, according to the intermediate value theorem, there exists an $a^*$ such that $F'(a^*)=0$.
\end{proof}

If there are $L$ local minimizers $x^1,\cdots,x^{L}$ in ${\cal B}_0$ whose criteria are bigger than $f(x_0)$, and a local minimizer $x'$ with smaller criterion outside ${\cal B}_0$. Denote $f_{\min}=\min_{i=1,...,L}\{f(x^i)\}$ , we have $L+1$ local maximizers $\xi_1<\cdots<\xi_{L+1}$. Since $f(x') < f(x_0)$, we can set $\delta_0$ such that $f_{\min}>f(x_0+\delta_0 d)$. Substituting $\xi_{L+1}$ to $\xi$ in the proof,  we can prove Theorem~\ref{existence}.



\section*{Appendix C}\label{fix_dem}
In the following, we will explain why $P_c>P_r$. In Alg.~\ref{alg:sampling with curiosity}, the main idea is using negative linear combination and adding a noise to make algorithm robustly. Now we will explain the insight of 'negative linear combination'. We first assume that there are two local minimizers. Without loss of generality, suppose that we are at a local minimum $x_0$, and there exists a local minima $x'$ ($f(x') < f(x_0)$). Then $x'$ has a neighborhood region $R_{x'}$, which satisfies $f(x)<f(x_0), \forall x\in R_{x'}$, then $x''$ denotes the center of circumscribed sphere of $R_{x'}$. Then $d^*\triangleq x''-x_0$ is called the central direction in the sequel. Further, we define the ray $ \ell_{d} =   x_0 + td, t > 0$. We have the following Lemma~\ref{sam}.



\begin{lemma}\label{sam}
Given an initial sample of directions and scores $\{(d_1, u_1)\cdots, (d_{N_0}, u_{N_0})\}$ ($N_0\leq n$), using negative linear combination of $d_1,\cdots,d_{N_0}$, it is of higher probability to obtain $d^*$ than that of the random sampling.
\end{lemma}
\begin{proof}
In the following, we first prove the theorem in case $n=2$. It is then generalized to $n > 2$.

In case $N_0=2$ and $n=2$, suppose at some time step, we have two linearly independent directions $d_1$ and $d_2$ with negative scores. Let $\Omega = \{x:\|x - x_0 \|_2 \leq M\}$ be the confined search space. The search space can then be divided into four regions $B^1, B^2, B^3$ and $B^*$\footnote{Namely, $B^1 = \{d = \alpha_1 d_1 + \alpha_2 d_2, \alpha_1 > 0, \alpha_2 < 0\}$, $B^2 = \{d = \alpha_1 d_1 + \alpha_2 d_2, \alpha_1 < 0, \alpha_2  > 0\}$, $B^3 = \{d = \alpha_1 d_1 + \alpha_2 d_2, \alpha_1 > 0, \alpha_2 > 0\}$.}. Particularly, $B^* = \{d = \alpha_1 d_1 + \alpha_2 d_2, \alpha_1 < 0, \alpha_2 < 0\}$. Suppose that $x'$ has a neighborhood region $R_{x'}$, which satisfies $f(x)<f(x_0), \forall x\in R_{x'}$, and the radius of the circumscribed sphere of $R_{x'}$ is $r_0$. The boundary of the circumscribed sphere and $x_0$ can form a cone $C^*$. By assumption, the lines $\ell_{d_i}, i = 1,2$ has no interaction with $C^*$ (otherwise we have found a direction that will lead to the attraction basin of $x')$, i.e.
\begin{eqnarray*}
 C^* \bigcap \{x | x\in \ell_{d_i}\}=\emptyset, \forall i \in \{1,2\}
\end{eqnarray*}
For each $d_i$, take $x^i$ and $x^*$ such that $x^i = \partial\Omega \bigcap\ell_{d_i}$ and $x^* = \partial\Omega \bigcap \ell_{d^*}$, respectively. Then the boundary of $B(x^i, r_1)$ and $x_0$ form a cone $C^i$, where
\begin{equation}\label{r1} r_1 = \max_{r}\{B(x^*,r)\subset C^*\} \end{equation} Let $\widetilde{C} = \bigcup_{i=1}^2 {C^i}$, we have
\begin{eqnarray*}
\widetilde{C} \bigcap \{x|x\in \ell_{d^*}\}=\emptyset
\end{eqnarray*}Thus, we should avoid looking for directions in the union of $C^i$'s.

Notice that if $B^*\bigcap C^i = \emptyset$ for $i = 1,2$. Denote $\widetilde{\Omega}=\left\{x|x\in \Omega,x\notin\widetilde{C}\right\}$, then $\tilde{P}_c$, the probability of finding $d^*$ in $\Omega$ by the negative linear combination, can be computed as follows:
\begin{eqnarray*}
\tilde{P}_c &=&	P\{\tilde{d} = d^*;x'' \in B^*\}+ P\{\tilde{d} = d^*;x'' \notin B^*\}\\
&=&  P\{x'' \in B^*\}P\{\tilde{d} = d^*|x'' \in B^*\}
\end{eqnarray*}
where $\tilde{d}$ is the direction got by negative linear combination. Notice that $ P\{\tilde{d}= d^* |x'' \in B^*\}=\frac{V_{d^*}}{V_{B^*}}$ where $V_{d^*}, V_{B^*}$ is the measure of $d^*,B^*$, respectively. Denote $\tilde{P}_r$ the probability of finding $d^*$ in $\Omega$ by random sampling, then $V_{d^*}$ can be represented by $\tilde{P}_r$ and the measure of $\Omega$. That is, $V_{d^*}=\tilde{P}_r \cdot V_{\Omega}$. As $d^*$ does not interact with $\tilde{C}$, thus $x''\notin \tilde{C}$. Then we have:
\begin{eqnarray*}
  \tilde{P}_c 	=  \frac{V_{B^*}}{V_{\widetilde{\Omega}}} \cdot \frac{\tilde{P}_r \cdot V_{\Omega}}{V_{B^*}} =  \frac{V_{\Omega}}{V_{\widetilde{\Omega}}}\cdot \tilde{P}_r>\tilde{P}_r
\end{eqnarray*}
where $V_{\Omega},V_{\widetilde{\Omega}}$ is the measure of $\Omega$ and $\widetilde{\Omega}$, respectively. The last inequality holds because $\widetilde{\Omega} \subset \Omega$.

If $B^*\bigcap C^i \neq \emptyset$ for $i = 1,2$, then $B^i$ is covered by $C^i$. Since the region covered by $C^i\ (i=1,2)$ in $B^3$ has a larger measure than $B^*$, $B^*$ is thus the best region for sampling.

In case $n > 2$, we have $n$ directions with negative scores $d_1,\cdots,d_n$. If set $d_2$ as the subspace $S=\{d = \sum_{i=2}^{n}\alpha_id_i,\ \alpha_i>0\}$, since $S$ has a zero measure in $\mathbb{R}^n$, the proof degenerates into the $n=2$ case.
\end{proof}

Furthermore, we will present why $P_c>P_r$:

We illustrate by using $C^2$ in Fig.~\ref{cone} in $n=2$.  In Fig.~\ref{cone}, $d_2$ is a direction with negative score. If we want to create $\tilde{d}$ which is a promising direction, then $d^*$ must be between $d_{\text{low}}$ and $d_{\text{up}}$ as shown in Fig.~\ref{cone2}.

To define $d_{\text{low}}$ and $d_{\text{up}}$, let $C_{d}$ is the cone made by the boundary of $B(x_d,r_1)$ and $x_0$ where $x_d=\partial\Omega\bigcap \ell_d $ for any $t > 0$ and $d$, and $r_1$ is defined in Eq.~\ref{r1}. We further define $D_{\tilde{d}} = \{d|\tilde{d}\in C_{d} \text{\  and\ } C_{d}\bigcap \ell_{d_2}=\emptyset\}$ and $R=\{l_d|d \in D_{\tilde{d}} \}$. $d_{\text{low}}$ and $d_{\text{up}}$ are considered as the ray from $x_0$ to the boundary of $R$. Similarly to the definition of $x_d$, we define $x_{\tilde{d}} = \partial\Omega\bigcap \ell_{\tilde{d}} $, $x_{d_{\text{up}}} = \partial\Omega\bigcap \ell_{{{d}_{\text{up}}}} $ and $x_{d_{\text{low}}} = \partial\Omega\bigcap \ell_{{{d}_{\text{low}}}}$.

If the distance between $x^2$ and $x_{\tilde{d}}$ is larger, then the distance between $x_{d_{\text{up}}}$ and $x_{d_{\text{low}}}$ must be larger as shown in Fig.~\ref{cone2}. This implies the probability that $\tilde{d}$ is promising is higher. When the distance between $x^2$ and $x_{\tilde{d}}$ is larger than $r_1$, i.e. $\tilde{d}\notin C^2$, the probability is the maximum since there is no $d$ such that $C_{d} \bigcap \ell_{d_2} \neq \emptyset $ and $\tilde{d} \in C_d$. Therefore, it is the best to use the opposite direction of $d_2$ since the point by interacting $-d_2$ and  $\Omega $ is the furthest to $x_2$.

Similarly for $C^1$, the best direction should be $-d_1$. Taking both $d_1$ and $d_2$ into consideration, a direction is promising only if its interaction point with $\Omega$ is the furthest to both $x^1$ and $x^2$. It is thus the best to sample a direction in the region spanned by $-d_1$ and $-d_2$, i.e. $B^*$.

For $n>2$, we have $n$ directions with negative scores. Given the $n$ directions, we can construct a spanned space $B = \{d = \sum_{i=1}^N \alpha_i d_i \}$. Depending on the signs of $\alpha_i$'s, we have $2^N$ sub-regions $B^i, i = 1, \cdots, 2^N$. We take $B^*$ be the region with all negative $\alpha_i$'s.

Similar to the analysis in $n=2$, for each $d_i$, the point $\Omega \bigcap \ell_{-d_i}$ is the furthest to $x^i$.  A direction is promising only if its interaction point with $\Omega$ is the furthest to all $x^i$'s. Therefore, $B^*$ is the best region for sampling among the $2^N$ regions.

A direction is sampled with equal probability in $\Omega$ in random sampling. On the contrary, using negative linear combination is sampling in $B^*$. Therefore, we have $P_c > P_r$.

If $f(x)$ has two local minima, we have explained $P_c>P_r$. In case $f(x)$ has 3 or more local minimizers, the sampling procedure can be done as follows. Assuming we have sampled $N$ directions, $\{d_i\}_{i=1}^{N}$, from which at least one local minimizer $x_{\text{last}}$ cannot be reached. It is not wise to sample within the cones induced by local minimizers we have visited. Instead, the negative rewards associated with these directions should be used as the linear combination for sampling directions for $x_{\text{last}}$. Therefore, this combination is guaranteed to be more efficient to sample promising directions for $x_{\text{last}}$ than random sampling.


\begin{figure}
\includegraphics[width = 1\columnwidth]{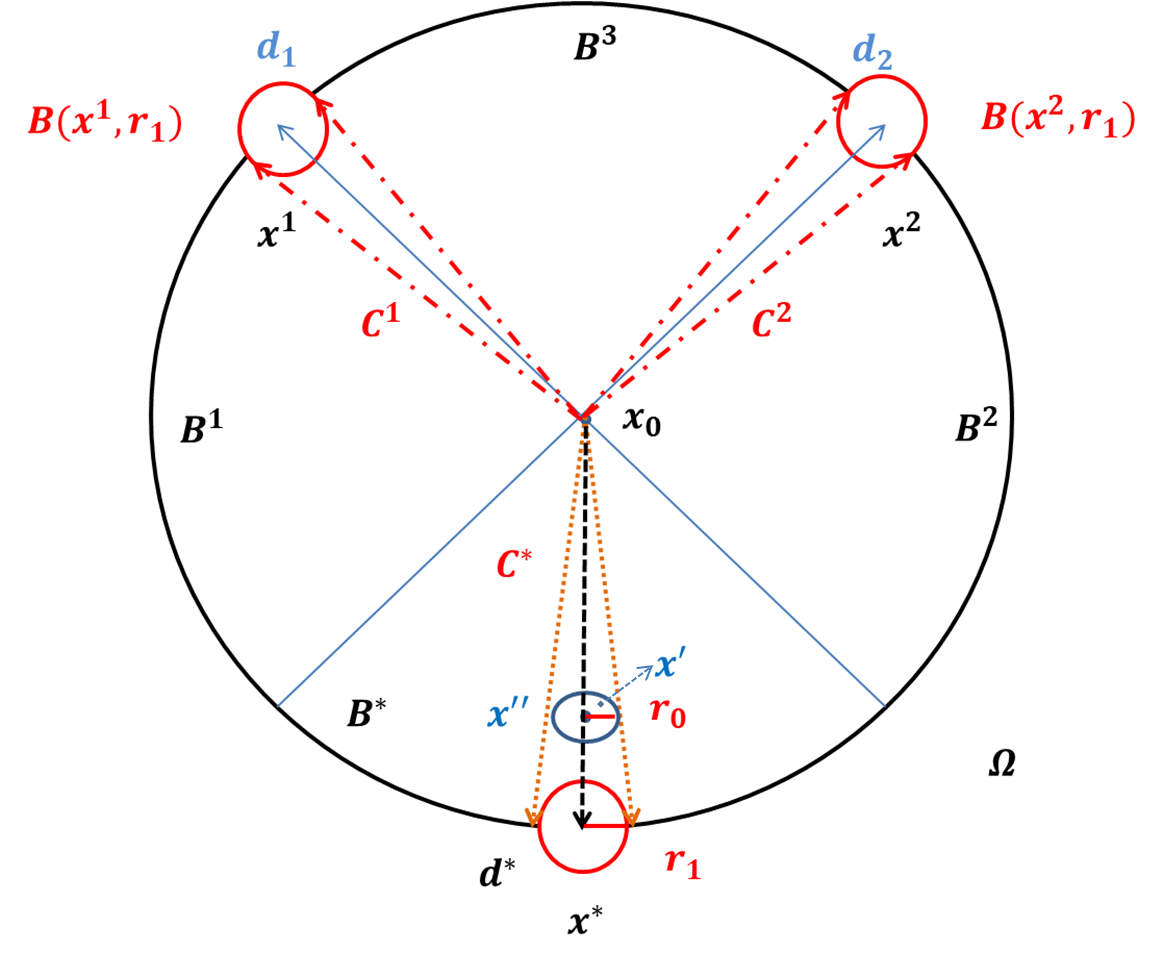}\centering
\caption{Illustration of Lemma~\ref{sam} in 2-D case. In the figure, $d^*$ is to be found in $\Omega$. $r_0$ is the radius of the circumscribed sphere of the attraction basin of $x'$. $x_0,B(x'',r_0)$ form a cone $C^*$. For each $d_i$, take $x^i \in\partial\Omega \bigcap d_i$, the boundary of $B(x^i, r_1)$ and $x_0$ forms a cone $C^i$. Then ${C^i} \bigcap \{x|x=x_0+td^*,x\in\Omega, t>0\}=\emptyset, \forall i$. It is clear that sampling a direction in $B^*$ is the best choice.}\label{cone}
\end{figure}
\begin{figure*}[htbp]
\centering
\subfigure[]{\includegraphics[width = 0.75\columnwidth]{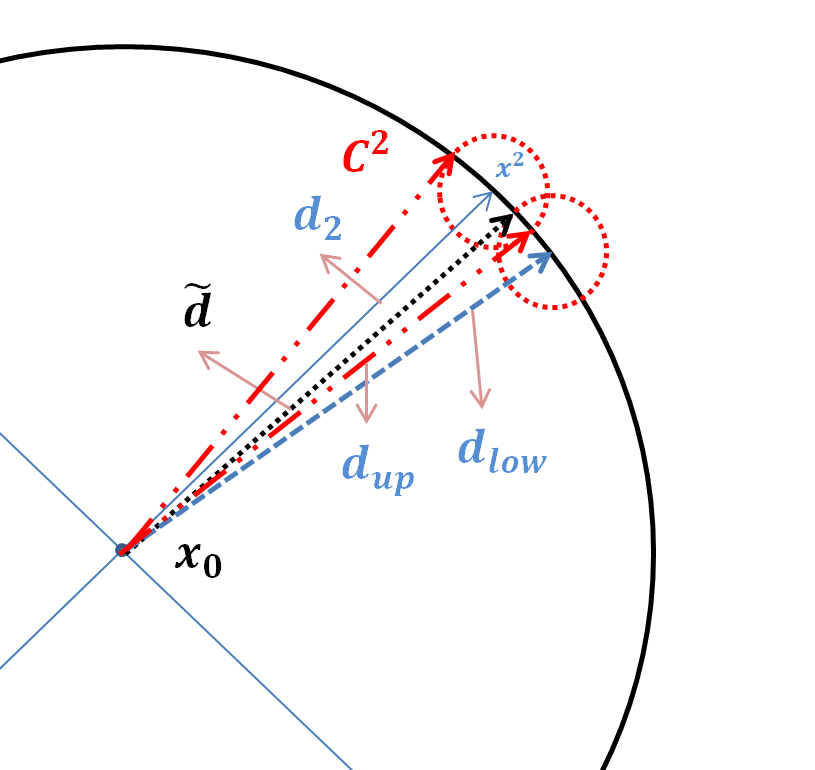}}
\subfigure[]{\includegraphics[width = 0.73\columnwidth]{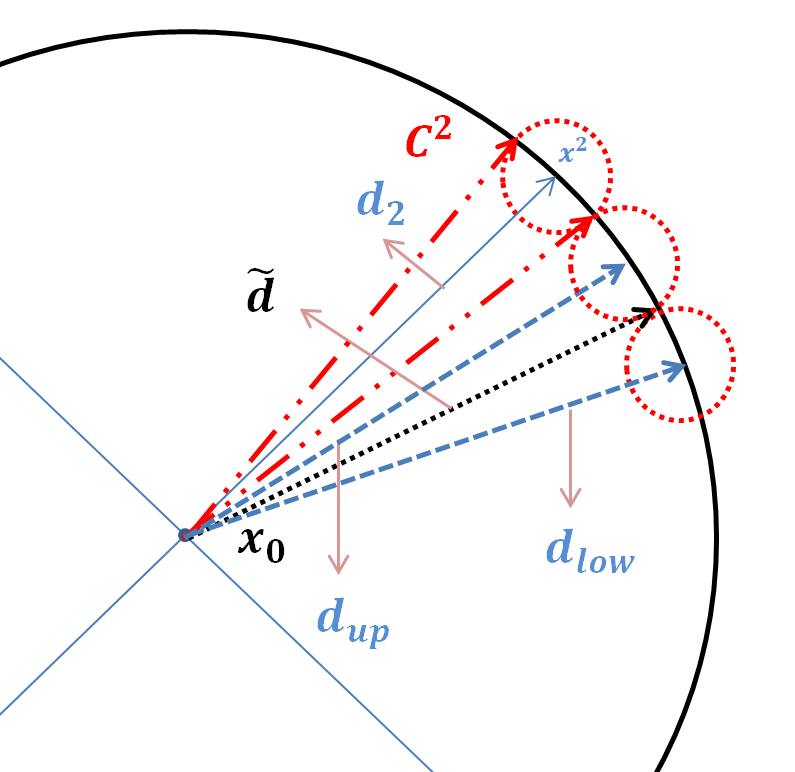}}
\caption{Demonstration of promising direction and optimal direction. (a) shows when $x^2$ and $x_{\tilde{d}}$ is close to each other, $d_{\text{up}}$ and $d_{\text{low}}$ are close too. (b) shows when the distance between $x^2$ and $x_{\tilde{d}}$ is bigger than $r_1$, the distance between $x_{d_{\text{up}}}$ and $x_{d_{\text{low}}} $ reaches the maximum. }\label{cone2}
\end{figure*}

\section*{Appendix D}

To train (test) the learned policy, the Gaussian mixture functions are used (cf. Eq.~\ref{mix_gauss}). And we use $\Sigma_1=diag\{1,1\}$, $\Sigma_2=diag\{1,1\}$, $\mu_1=[0,0]^\intercal$; $\mu_2=[5,5]^\intercal.$ for 2-D problem. When testing, $\Sigma_1=diag\{1,1,1,8,8\}$, $\Sigma_2=diag\{1,1,1,8,8\}$, $\mu_1=[0,0,0,0,0]^\intercal$, $\mu_2=[-5,-5,-5,-5,-5]^\intercal$ for 5-D problem. When training, the following settings with different means and covariances, are applied in Table~\ref{rl_sample}.

\begin{itemize}
\item order 1-4, problem is in 2-D, $\Sigma_1=diag\{1,8\}$, $\Sigma_2=diag\{1,3\}$, $\mu_1=[0,0]^\intercal$; $\mu_2=[7,7]^\intercal,\ [5,7]^\intercal,\ [3,7]^\intercal,\ [4,7]^\intercal$ respectively;
\item order 5, problem is in 5-D, $\Sigma_1=diag\{1,1,1,8,8\}$, $\Sigma_2=diag\{1,1,1,8,8\}$, $\mu_1=[0,0,0,0,0]^\intercal$, $\mu_2=[5,5,5,5,5]^\intercal$;
\item order 6, problem is in 5-D, $\Sigma_1=diag\{1,1,1,8,8\}$, $\Sigma_2=diag\{1,1,1,8,8\}$, $\mu_1=[0,0,0,0,0]^\intercal$, $\mu_2=[4,4,5,5,5]^\intercal$;
\item order 7, problem is in 5-D, $\Sigma_1=diag\{1,1,1,8,8\}$, $\Sigma_2=diag\{1,1,1,8,8\}$, $\mu_1=[0,0,0,0,0]^\intercal$, $\mu_2=[3,3,5,5,5]^\intercal$;
\item order 8, problem is in 5-D, $\Sigma_1=diag\{1,1,1,1,1\}$, $\Sigma_2=diag\{1,1,1,8,8\}$, $\mu_1=[0,0,0,0,0]^\intercal$, $\mu_2=[5,5,5,5,5]^\intercal$;
\item order 9, problem is in 5-D, $\Sigma_1=diag\{1,1,1,1,1\}$, $\Sigma_2=diag\{1,1,1,8,8\}$, $\mu_1=[0,0,0,0,0]^\intercal$, $\mu_2=[3,3,5,5,5]^\intercal$;
\end{itemize}

\end{document}